\documentclass[11pt]{article}

\usepackage{dsfont}

\def\colorful{0}

\usepackage{amsthm,amsfonts,amsmath,amssymb,epsfig,color,float,graphicx,verbatim, enumitem,empheq}
\usepackage{multirow}
\usepackage{algorithm}
\usepackage[noend]{algpseudocode}

\newif\ifhyper\IfFileExists{hyperref.sty}{\hypertrue}{\hyperfalse}
\hypertrue
\ifhyper\usepackage{hyperref}\fi

\usepackage{enumitem}

\usepackage{framed}
\usepackage{nicefrac}

\def\nnewcolor{0}
\ifnum\nnewcolor=1
\newcommand{\nnew}[1]{{\color{red} #1}}
\fi
\ifnum\nnewcolor=0
\newcommand{\nnew}[1]{#1}
\fi

\ifnum\colorful=1
\newcommand{\new}[1]{{\color{red} #1}}

\else
\newcommand{\new}[1]{{#1}}

\fi

\newtheorem{theorem}{Theorem}[section]

\newtheorem{lemma}[theorem]{Lemma}
\newtheorem{informal theorem}[theorem]{Theorem (informal statement)}

\newtheorem{proposition}[theorem]{Proposition}
\newtheorem{corollary}[theorem]{Corollary}
\newtheorem{claim}[theorem]{Claim}
\newtheorem{fact}[theorem]{Fact}

\theoremstyle{definition}
\newtheorem{definition}[theorem]{Definition}
\newcommand{\eqdef}{\stackrel{{\mathrm {\footnotesize def}}}{=}}

\newcommand{\bx}{\mathbf{x}}
\newcommand{\by}{\mathbf{y}}
\newcommand{\bz}{\mathbf{z}}
\newcommand{\bv}{\mathbf{v}}
\newcommand{\bw}{\mathbf{w}}

\newcommand{\R}{\mathbb{R}}

\newcommand{\Z}{\mathbb{Z}}
\newcommand{\N}{\mathbb{N}}
\newcommand{\eps}{\epsilon}

\newcommand{\poly}{\mathrm{poly}}

\DeclareMathOperator*{\pr}{\mathbf{Pr}}
\DeclareMathOperator*{\E}{\mathbf{E}}
\renewcommand{\Pr}{\pr}

\newcommand{\sgn}{\mathrm{sign}}
\newcommand{\sign}{\mathrm{sign}}

\newcommand{\opt}{\mathrm{OPT}}
\newcommand{\D}{\mathcal{D}}

\newcommand{\relu}{\mathrm{ReLU}}

\newcommand{\be}{\mathbf{e}}

\newcommand{\dotp}[2]{\langle #1, #2 \rangle}

\newcommand{\x}{\mathbf{x}}
\renewcommand\vec[1]{\mathbf{#1}}
\def\d{\mathrm{d}}
\newcommand\snorm[2]{\left\| #2 \right\|_{#1}}
\newcommand{\normal}{\mathcal{N}}
\newcommand{\normald}[1]{\mathcal{N}_{#1}}
\newcommand{\normalone}{\mathcal{N}}
\newcommand{\gns}{\mathrm{GNS}}
\newcommand{\lspace}[1]{L^{#1}}
\newcommand{\matr}{\vec}
\renewcommand{\top}{\intercal}

\newcommand{\citet}{\cite}
\newcommand{\citep}{\cite}
\title{The Optimality of Polynomial Regression\\ 
for Agnostic Learning under Gaussian Marginals}

\author{
Ilias Diakonikolas\thanks{Supported by NSF Award CCF-1652862 (CAREER), a Sloan Research Fellowship, and a DARPA  Learning with Less Labels (LwLL) grant.}\\
University of Wisconsin Madison\\
{\tt ilias@cs.wisc.edu}\\
\and
Daniel M. Kane\thanks{Supported by NSF Award CCF-1553288 (CAREER) and a Sloan Research Fellowship.}\\
University of California, San Diego\\
{\tt dakane@cs.ucsd.edu}\\
\and
 Thanasis Pittas  \\
University of Wisconsin Madison\\
{\tt pittas@wisc.edu}\\
\and
 Nikos Zarifis\thanks{Supported in part by a DARPA  Learning with Less Labels (LwLL) grant.} \\ University of Wisconsin Madison\\
{\tt zarifis@wisc.edu}\\
}

\begin{document}

\maketitle

\begin{abstract}
We study the problem of agnostic learning under the Gaussian
distribution. We develop a method for finding hard families 
of examples for a wide class of problems by using LP duality. 
For Boolean-valued concept classes, we show that the 
$L^1$-regression algorithm is essentially best possible,
and therefore that the computational difficulty of agnostically learning a concept class
is closely related to the polynomial degree required to approximate any 
function from the class in $L^1$-norm. Using this characterization along with additional
analytic tools, we obtain optimal SQ lower bounds for agnostically learning linear
threshold functions and the first non-trivial SQ lower bounds for
polynomial threshold functions and intersections of halfspaces.
We also develop an analogous theory for agnostically learning
real-valued functions, and as an application prove near-optimal SQ lower
bounds for agnostically learning ReLUs and sigmoids.
\end{abstract}

\setcounter{page}{0}

\thispagestyle{empty}

\newpage

\section{Introduction} \label{sec:intro}

\subsection{Background and Motivation} \label{ssec:background}

In Valiant's Probably Approximately Correct (PAC) learning model~\cite{val84},
a learner is given access to random examples that are consistently labeled
according to an unknown function in the target concept class. Here we focus 
on the {\em agnostic framework}~\cite{Haussler:92, KSS:94}, 
which models learning in the presence of worst-case noise.
Roughly speaking, in the agnostic PAC model, 
we are given i.i.d.\ samples from a joint distribution $D$ on labeled examples $(\bx, y)$, 
where $\bx \in \R^n$ is the example and $y \in \R$ is the corresponding label, 
and the goal is to compute a hypothesis that is competitive with the ``best-fitting'' function
in the target class $\mathcal{C}$. The notion of agnostic learning 
is meaningful both for learning Boolean-valued functions (under the 0-1 loss) and 
for learning real-valued functions (typically, under the $L^2$-loss). For concreteness,
we restrict the proceeding discussion to the Boolean-valued setting.

In the distribution-independent setting, agnostic learning 
is known to be computationally hard, even for simple concept classes
and weak learning~\cite{GR:06, FGK+:06short, Daniely16}.
On the other hand, under distributional assumptions, 
efficient learning algorithms with worst-case noise are possible. 
A line of work~\cite{KKMS:08, KLS09, ABL17, Daniely15, DKS18-nasty, DKTZ20-agnostic} 
has given efficient learning algorithms in the agnostic model for natural concept classes 
and distributions with various time-accuracy tradeoffs. 
In this paper, we will focus on agnostic learning under the Gaussian distribution on examples.
For Boolean-valued concept classes, we have the following definition.

\begin{definition}[Agnostic Learning Boolean-valued Functions with Gaussian Marginals] \label{def:agn-bool}
Let $\mathcal{C}$ be a class of Boolean-valued concepts on $\R^n$.
Given i.i.d.\ samples $(\bx, y)$ from a distribution $D$ on $\R^n \times \{\pm 1\}$, where
the marginal $D_{\bx}$ on $\R^n$ is the standard Gaussian $\mathcal{N}_n$ 
and no assumptions are made on the labels $y$, the goal is to output a hypothesis 
$h: \R^n \to \{\pm 1\}$ such that with high probability we have 
$\pr_{(\bx, y) \sim D} [h(\bx) \neq y] \leq \opt+\eps$,
where $\opt  = \inf_{f \in \mathcal{C}} \pr_{(\bx, y) \sim D} [f(\bx) \neq y]$.
\end{definition}

The only known algorithmic technique for agnostic learning in the setting of Definition~\ref{def:agn-bool}
is the $L^1$-polynomial regression algorithm~\cite{KKMS:08}. This algorithm uses
linear programming to compute a low-degree polynomial that minimizes the $L^1$-distance 
to the target function. Its performance hinges on how well the underlying
concept class $\mathcal{C}$ can be approximated, in $L^1$-norm, by low-degree polynomials. 
In more detail,
if $d$ is the (minimum) degree such that any $f \in \mathcal{C}$ can be $\eps$-approximated
(in $L^1$-norm) by a degree-$d$ polynomial, the algorithm has sample complexity and running time 
$n^{O(d)}/\poly(\eps)$ and outputs a hypothesis with misclassification error $\opt+\eps$.

For several natural concept classes and distributions on examples, 
the aforementioned degree $d$ is independent of the dimension $n$, 
and only depends on the error $\eps$ (and potentially other size parameters).
For these settings, the $L^1$-regression algorithm can be viewed as a polynomial-time
approximation scheme (PTAS) for agnostic learning. Examples of such concept classes include
Linear Threshold Functions (LTFs)~\cite{KKMS:08, DGJ+:10, DKN10}, 
Bounded Degree Polynomial Threshold Functions (PTFs)~\cite{DHK+:10, Kane:10, DRST14, HKM14}, 
Intersections of Halfspaces~\cite{KKMS:08, KOS:08, Kane14}, and other geometric concepts~\cite{KOS:08}. 
Specifically, for the class of LTFs under the Gaussian distribution, the $L^1$-regression algorithm 
is known to have sample and computational complexity of $n^{O(1/\eps^2)}$.

For each of the above concept classes, $L^1$-polynomial regression is the fastest (and, essentially, the only) 
known agnostic learner. It is natural to ask whether there exists an agnostic learner with 
significantly improved sample/computational complexity.
\begin{center}
{\em Can we beat $L^1$-polynomial regression for agnostic learning under Gaussian marginals?}
\end{center}
As our first main contribution, we answer the above question in the negative for all concept classes
satisfying some mild properties (including all the geometric concept classes mentioned above).
Our lower bound applies for the class of Statistical Query (SQ) algorithms.
Statistical Query (SQ) algorithms are a class of algorithms that are allowed 
to query expectations of bounded functions of the underlying distribution
rather than directly access samples. 
Formally, an SQ algorithm has access to the following oracle.

\new{
\begin{definition}[\textsc{STAT} Oracle] \label{def:stat-oracle}
Let $D$ be a distribution on labeled examples supported on $X \times [-1, 1]$,
for some domain $X$.
A statistical query is a function $q: X \times [-1, 1] \to [-1, 1]$.  We define
\textsc{STAT}$(\tau)$ to be the oracle that given any such query $q(\cdot, \cdot)$
outputs a value $v$ such that
$|v - \E_{(\vec x, y) \sim D}\left[q(\vec x, y)\right]| \leq \tau$,
where $\tau>0$ is the tolerance parameter of the query.
\end{definition}
}

\noindent The SQ model was introduced by~\cite{Kearns:98} 
as a natural restriction of the PAC model~\cite{val84} 
and has been extensively studied in learning theory; 
see, e.g.,~\cite{FGR+13, FeldmanPV15, FeldmanGV17, Feldman17} for some recent references. 
The reader is referred to~\cite{Feldman16b} for a survey.
The class of SQ algorithms is fairly broad: a wide range of known algorithmic techniques 
in machine learning are known to be implementable using SQs
(see, e.g.,~\cite{Chu:2006, FGR+13, FeldmanGV17}).

Returning to our agnostic learning setting,  roughly speaking, 
{\em we show that a lower bound of $d$ on the degree of any $L^1$ approximating polynomial
can be translated to an SQ lower bound of $n^{\Omega(d)}$ for the agnostic learning problem.}
This lower bound is tight, since the $L^1$-regression algorithm 
can be implemented in the SQ model with complexity $n^{O(d)}$.

We note that a similar characterization had been previously shown, 
under somewhat different assumptions,
for agnostic learning under the uniform distribution on the hypercube~\cite{DFTWW15}. 
We explain the technical differences and similarities with our results in Section~\ref{ssec:related}. 
It is worth pointing out that learning under the Gaussian distribution is generally believed to be 
computationally easier than learning under the uniform distribution on the hypercube in a number of settings. 
For example, prior work~\cite{ABL17, DKS18-nasty, DKTZ20-agnostic} has given ``constant factor'' agnostic learners 
for LTFs on $\R^n$ under the Gaussian distribution --- i.e., algorithms with error $O(\opt)+\eps$ --- 
that run in $\poly(n/\eps)$ time. \new{No polynomial time algorithm with such an error guarantee 
is known for any discrete distribution. At a high-level, known algorithms 
for these problems make essential use of the {\em anti-concentration} of the Gaussian distribution, 
which fails in the discrete setting.}
Similar algorithmic gaps exist for robustly learning low-degree PTFs 
and intersections of halfspaces~\cite{DKS18-nasty}.

Our generic lower bound result for the Boolean case (Theorem~\ref{thm:main-boolean}) 
reduces the problem of proving explicit SQ lower bounds for agnostic learning 
to the structural question of proving lower bounds on the $L^1$ polynomial approximation degree 
(under the Gaussian measure). As our second contribution, 
we provide a toolkit to prove explicit degree lower bounds. 
As a corollary, we prove optimal or near-optimal SQ lower bounds for various natural classes,
including LTFs, PTFs, and intersections of halfspaces.

Moving away from the Boolean-valued setting, an interesting direction is to understand
the complexity of agnostic learning for real-valued function classes. In recent years, 
this broad question has been intensely investigated in learning theory, 
in part due to its connections to deep learning. 
Here we focus on agnostic learning under the $L^2$-loss.

\begin{definition}[Agnostic Learning Real-valued Functions with Gaussian Marginals] \label{def:agn-real}
Let $\mathcal{C}$ be a class of real-valued concepts on $\R^n$.
Given i.i.d.\ samples $(\bx, y)$ from a distribution $D$ on $\R^n \times \R$, where
the marginal $D_{\bx}$ on $\R^n$ is the standard Gaussian $\mathcal{N}_n$ 
and no assumptions are made on the labels $y$, the goal is to output a hypothesis 
$h: \R^n \to \R$ such that with high probability we have 
$\E_{(\bx, y) \sim D} [(h(\bx) - y)^2]^{1/2} \leq \opt+\eps$,
where $\opt  = \inf_{f \in \mathcal{C}} \E_{(\bx, y) \sim D}[(f(\bx) - y)^2]^{1/2}$.
\end{definition}

A prototypical concept class of significant recent interest are Rectified Linear Units (ReLUs). 
A ReLU is any real-valued function $f: \R^n \to \R_+$ of the form 
$f(\bx) = \relu \left(\langle \bw, \bx \rangle + \theta \right)$, $\bw \in \R^n$ and $\theta \in \R$,
where $\relu: \R \to \R_+$ is defined as $\relu(u)= \max \{ 0, u\}$.
ReLUs are the most commonly used activation functions in modern deep neural networks.
The corresponding agnostic learning problem 
is a fundamental primitive in the theory of neural networks that has been extensively studied
in recent years~\cite{GoelKKT17, MR18, GoelKK19, DGKKS20, GGK20, DKZ20}.

Our techniques extend to real-valued concepts leading to improved and nearly tight SQ lower
bounds for natural concept classes. We describe our contributions in the following subsection.

\subsection{Our Contributions} \label{ssec:results}

\paragraph{Contributions for Boolean-valued Concepts}
Our main general result for Boolean-valued concepts is the following:

\begin{theorem}[Generic SQ Lower Bound, Boolean Case] \label{thm:main-boolean}
Let $n,m \in \Z_+$  with $m\leq n^a$ for any constant $0<a<1/2$ and $\eps \geq n^{-c}$ 
for some sufficiently small constant $c>0$.
Fix a function $f:\R^m\to \{\pm 1\}$.
Let $d$ be the smallest integer such that there exists a degree at most $d$ polynomial
$p : \R^m \to \R$ satisfying $ \E_{\x\sim \normald{m}}[|p(\x)-f(\x)|] < 2\eps$.
Let $\mathcal{C}$ be a class of Boolean-valued functions on $\R^n$
which includes all functions of the form ${F}(\x) = f(\vec P\x)$,
for any $\vec P \in \R^{m \times n}$ such that $\vec{P}\vec{P}^\top=\vec{I}_m$.
Any SQ algorithm that agnostically learns $\mathcal{C}$ under $\normald{n}$
to error $\opt+\eps$ either requires queries with tolerance at most $n^{-\Omega(d)}$
or makes at least $2^{n^{\Omega(1)}}$ queries.
\end{theorem}

\noindent The $L^1$-polynomial regression algorithm and
Theorem~\ref{thm:main-boolean} characterize the complexity of agnostic
learning under the Gaussian distribution -- within the class of SQ algorithms --
for a range of concept classes. If $d$ is the (minimum) degree for which 
any function in $\mathcal{C}$ can be $\eps$-approximated by a degree-$d$ 
polynomial in $L^1$-norm, the complexity of agnostically learning $\mathcal{C}$ 
is, roughly, $n^{\Theta(d)}$. 

\medskip 
\noindent {\bf \em Applications of Theorem~\ref{thm:main-boolean}.}
Note that the above result does not tell us what  the optimal degree $d$ is 
for any given concept class $\mathcal{C}$.
Using analytic techniques, we establish explicit lower bounds on the $L^1$ polynomial
approximation degree for three fundamental concept classes: Linear Threshold
Functions (LTFs), Polynomial Threshold Functions, and Intersections of Halfspaces.
As a corollary, we obtain explicit SQ lower bounds for these classes.
Our applications are summarized in Table~\ref{tab:intro_distribution_classes}.

\begin{table}[htpb]
\def\arraystretch{1.2}
\centering

\begin{tabular}{l c c}
\hline \hline
\textbf{Concept Class}         & Lower Bound                                                                           & Upper Bound                                                             \\
\hline \hline
LTFs                        &$ \Omega \left(1/\eps^2 \right)$~\cite{ganzburg2002limit}           & $O \left( 1/\eps^2 \right)$~\cite{ganzburg2002limit, DKN10}  \\
Degree-$k$ PTFs                & $\Omega \left( k^2/\eps^2 \right)$ (Thm~\ref{thm:ptfs-degree})                    & $O \left(k^2/{\eps^4} \right)$~\cite{Kane:10}                           \\
Intersections of $k$ Halfspaces & $\tilde{\Omega} \left(\sqrt{\log k}/{\eps} \right)$ (Thm~\ref{thm:inters-degree})
                               & $O \left(\log k/{\eps^4} \right)$~\cite{KOS:08}                                                                                                                \\
\hline
\end{tabular}

\caption{Bounds on the degree $d$ of $\eps$-approximating polynomials in $\lspace{1}$-error
    under the Gaussian measure. For each concept class, we obtain an SQ lower bound of $n^{\Omega(d)}$.}
\label{tab:intro_distribution_classes}
\end{table}

For the class of LTFs, using a known degree 
lower bound for the sign function~\cite{ganzburg2002limit}, we immediately
obtain an SQ lower bound of $n^{\Omega(1/\eps^2)}$. This bound is optimal (within polynomial factors),
improving on the previous SQ lower bound of $n^{\Omega(1/\eps)}$~\cite{GGK20, DKZ20}.
Our approach is simpler and more general compared to these prior works, 
immediately extending to other families.
For the broader class of degree-$k$ PTFs, we establish a degree lower bound of 
$\Omega(k^2/\eps^2)$  (Proposition~\ref{prop:ptfs-degree}), which yields 
an SQ lower bound of $n^{\Omega(k^2/\eps^2)}$ for the agnostic learning problem.

Our third explicit degree lower bound is for intersections of $k$ halfspaces.
For this concept class, we prove a degree lower bound of $d = \tilde{\Omega}(\sqrt{\log k}/{\eps})$ 
which implies a corresponding SQ lower bound of $n^{\tilde{\Omega}(\sqrt{\log k}/{\eps})}$.
In the process, we establish a new structural result translating lower bounds 
on the Gaussian Noise Sensitivity (GNS) of {\em any} Boolean function
to the $L^1$-polynomial approximation degree of the same function. 

Recall that the Gaussian Noise Sensitivity (GNS) of a function $f: \R^n \to \{\pm 1\}$ is defined
as $\gns_\rho(f) \eqdef \pr_{(\bx,\by)\sim \normald{n}^\rho}[f(\bx) \neq f(\by)]$, 
where $\normald{n}^\rho$ is the distribution of a $(1-\rho)$-correlated 
Gaussian pair (i.e., $\bx$ and $\vec y$ are standard Gaussians with correlation $(1-\rho)$).
We show the following:

\begin{theorem}[Structural Result] \label{thm:bound-l1-by-gns}
Let $f: \R^n \to \{\pm 1\}$ and $p:\R^n \to \R$ be a degree at most $d$ polynomial. Then,
we have that $\E_{\bx \sim \normald{n}}[|f(\bx)-p(\bx)|] \geq \Omega(1/\log(d))\gns_{(\log(d)/d)^2}(f)$.
Furthermore, for any $\eps>0$, we have that
$\E_{\bx \sim \normald{n}}[|f(\bx)-p(\bx)|] \geq \gns_\eps(f)/4 - O(d\sqrt{\eps})$.
\end{theorem}

\paragraph{Contributions for Real-valued Concepts}
For agnostically learning real-valued concepts,
we provide two generic lower bound results, analogous to Theorem~\ref{thm:main-boolean},
for Correlational SQ (CSQ) algorithms and general SQ algorithms respectively. 
A conceptual message of our results is that $L^2$ regression is essentially optimal against
CSQ algorithms, but not necessarily optimal against general SQ algorithms.

\medskip

\new{Recall that Correlational SQ (CSQ) algorithms are a subclass of SQ algorithms, 
where the algorithm is allowed to choose any bounded query function on the examples 
and obtain estimates of its correlation with the labels. (See Appendix~\ref{ssec:app-CSQ} for a detailed description.)
This class of algorithms is fairly broad, capturing many learning algorithms used in practice (including gradient-descent).}
For CSQ algorithms, we prove.

\begin{theorem}[Generic CSQ Lower Bound, Real-valued Case]\label{them:real-val-lb-csq}
Let $n,m \in \Z_+$  with $m\leq n^a$ for any constant $0<a<1/2$ and $\eps \geq n^{-c}$ 
for some sufficiently small constant $c>0$. Let $f :\R^m\to\R$ with $\E_{\bx \sim \normald{m}}[f^2(\bx)]=1$ 
and $d$ be the smallest integer such that there exists a degree at most $d$ polynomial 
$p: \R^m \to \R$ satisfying $\snorm{2}{f-p} < \eps$. Let $\cal C$ be a class of real-valued functions 
on $\R^n$ which includes all functions of the form $F(\bx) = f(\vec P \vec x)$, for any  matrix 
$\vec P \in \R^{m\times n}$ satisfying $\vec{P}\vec{P}^\top = \vec{I}_m$.  Then, any CSQ algorithm
that agnostically learns $\mathcal{C}$ over $\normald{n}$ to $L^2$-error $\opt+\eps$ either 
requires queries with tolerance at most $n^{-\Omega(d)}$ or makes at least $2^{n^{\Omega(1)}}$ queries.
\end{theorem}

Our lower bound for the general SQ model is presented below. 
The difference between the two is that the latter uses the $L^1$-norm to measure 
the approximation of $f$ by polynomials.

\begin{theorem}[Generic SQ Lower Bound, Real-valued Case]\label{them:real-val-lb-sq}
Let $n,m \in \Z_+$  with $m\leq n^a$ for any constant $0<a<1/2$ and $\eps \geq n^{-c}$ 
for some sufficiently small constant $c>0$. Let $f :\R^m\to\R$ with $\E_{\bx \sim \normald{m}}[f^2(\bx)]=1$ 
and $d$ be the smallest integer such that there exists a degree at most $d$ polynomial 
$p: \R^m \to \R$ satisfying $\snorm{1}{f-p} <  \eps$. Let $\cal C$ be a class of real-valued functions 
on $\R^n$ which includes all functions of the form $F(\bx) = f(\vec P \vec x)$, 
for any  matrix $\vec P \in \R^{m\times n}$ satisfying $\vec{P}\vec{P}^\top = \vec{I}_m$.  
Then, any SQ algorithm  that agnostically learns $\mathcal{C}$ over $\normald{n}$ 
to $L^2$-error $\opt+\eps$ either requires queries with tolerance 
at most $n^{-\Omega(d)}$ or makes at least $2^{n^{\Omega(1)}}$ queries.
\end{theorem}

\medskip 
\noindent {\bf \em Applications of Theorems~\ref{them:real-val-lb-csq} and~\ref{them:real-val-lb-sq}.}
As in the Boolean-valued setting, obtaining explicit (C)SQ lower bounds for agnostically learning
real-valued concepts requires analytic tools to establish lower bounds on the degree of polynomial approximations.
In this paper, we give such lower bounds for two fundamental concept classes: ReLUs and sigmoids.
Establishing degree lower bounds for other non-linear activations is left as a question for future work.

\begin{table}[ht]
\center
\begin{tabular}{l c c c c c}
\hline \hline
&\multicolumn{2}{c}{$p=1$} &\multicolumn{2}{c}{$p=2$} \\
\textbf{Concept Class}  &Lower Bound &Upper Bound  &Lower Bound &Upper Bound  \\
\hline \hline
ReLUs                  & $ \Omega \left(1/\eps \right)$ (Cor.~\ref{cor:relu-lb})  
                             & $O \left( 1/\eps \right)$ & $\Omega \left( 1/{\eps}^{4/3} \right)$ (Cor.~\ref{cor:relu-lb}) & $O \left( 1/{\eps}^{4/3} \right)$  \\
Sigmoids               & $\Omega(\log(1/\eps ))$  (Thm.~\ref{thm:sigmoid-l1})                 & $O(\log^2(1/\eps ))$  
                             & $\Omega \left( \log^2(1/\eps) \right)$ (Cor.~\ref{cor:sigmoid-l2})     & $O \left( \log^2(1/\eps) \right)$ \\
\end{tabular}

\caption{Bounds on the degree $d$ of $\eps$-approximating polynomials in $\lspace{1}$ and $L^2$-error
under the Gaussian measure. For each concept class, we obtain a CSQ (resp. SQ) lower bound of $n^{\Omega(d)}$, 
where $d$ is the $L^2$ degree (resp. $L^1$ degree).}
\label{tab:real-valued-app}
\end{table}

Our degree lower bounds applications for both $L^1$ and $L^2$ polynomial approximations 
are summarized in Table~\ref{tab:real-valued-app}.
Combining these degree lower bounds  Theorems~\ref{them:real-val-lb-csq} and~\ref{them:real-val-lb-sq} 
implies explicit SQ lower bounds for ReLUs and sigmoids. 

Concretely, for agnostically learning ReLUs, we establish a CSQ lower bound of  
$n^{\Omega(1/\eps^{4/3})}$ (matching the $n^{O(1/\eps^{4/3})}$ upper bound obtained via $L^2$-regression);
and an SQ lower bound of $n^{\Omega(1/\eps)}$, improving on the previous best bound 
of $n^{\Omega((1/\eps)^{1/36})}$~\cite{GGK20, DKZ20}.

\subsection{Overview of Techniques} \label{ssec:techniques}

\paragraph{SQ Lower Bounds for Boolean-valued Functions}
The starting point for our lower bounds is the work of~\cite{DKS17-sq}, which
shows that if $D$ is a univariate distribution whose low-degree moments match those
of a standard Gaussian (and which satisfies some other mild niceness
conditions), then it is SQ-hard to distinguish between a standard
multivariate Gaussian and a distribution that is a copy of $D$ 
in a random direction and a standard Gaussian in the orthogonal directions. 
(This is shown in~\cite{DKS17-sq} for $D$ a $1$-dimensional distribution, 
but it is not hard to generalize to higher dimensional distributions.)

Note that the above setting is unsupervised.
To go from distributions to functions, 
we will try to produce a Boolean function $f$ of a few variables 
such that the distributions of $X$ conditioned on $f(X)=1$ and on $f(X)=-1$ 
match moments with a Gaussian. We generalize the techniques 
of~\cite{diakonikolas2020algorithms} to show that such
a function $f$ embedded in a hidden {\em low-dimensional subspace}
is SQ hard to distinguish from a random function. Our goal then is to find 
such a function $f$ that is $(1/2-\eps)$-close to a function in our family. 
Given this construction, learning the function to error $\opt+\eps/2$  
requires being able to distinguish $f$ from a random function.

The aforementioned approach was recently used by~\cite{DKZ20}. 
However, while that work constructs the function $f$
somewhat directly, here we take a more general approach. 
In more detail, it is not hard to phrase the conditions that (1) $f$ is bounded in $[-1, 1]$,
(2) it matches moments with low-degree polynomials, 
and (3) is not too far from the function we are trying to learn, 
as an infinite-dimensional linear program (LP). We can then non-constructively attempt to find 
the optimal value of such an LP by duality. We note that ``LP duality'' in this setting
is non-trivial -- we require some (basic) functional analysis tools 
to show that duality applies for the LPs we are considering on function spaces. 
Given this, we find that the dual program is equivalent to finding 
a low-degree polynomial that approximates the function we are trying to learn in $L^1$-norm. 
The degree of such a polynomial conveniently matches the parameter that determines 
the runtime of the $L^1$- polynomial regression algorithm. 
We can thus show that, for reasonable function families, 
the $L^1$-regression algorithm is in fact optimal, 
among SQ algorithms, up to polynomial factors.

The above characterization allows us to determine the complexity 
of agnostically learning LTFs, by leverage tight degree lower bounds for the sign function. 
For the cases of degree-$k$ PTFs and intersections of $k$ halfspaces, 
we do not know what the correct answer is, but we are able to prove 
non-trivial, and qualitatively close to optimal, lower bounds. 

\new{We note that the $L^1$ approximation theory for these functions is more challenging 
than the $L^2$ approximation theory (which is entirely determined by the Fourier
decay). To that end, we develop new techniques relating $L^1$
approximability to the Gaussian Noise sensitivity (Theorem~\ref{thm:bound-l1-by-gns}), 
which allows us to prove the first non-trivial lower bounds.
The proof of Theorem~\ref{thm:bound-l1-by-gns} works via a symmetrization technique. 
In particular, let $\theta = \arccos(1-\eps)$ and let $X$ and $Y$ be standard
Gaussians. Let $F_{X,Y}(\phi) := f(\sin(\phi) X + \cos(\phi)Y)$. 
Then we can write $\gns_{\eps}(f) = \Pr[F_{X,Y}(\phi) \neq F_{X,Y}(\phi+\theta)]$. 
On the other hand, $\|f-p\|_1 = \E[|F_{X,Y}(\phi) - p(\sin(\phi) X + \cos(\phi) Y)|]$. 
Thus, it suffices to show that if $F$ is {\em any} Boolean function on the circle 
that the $L^1$ approximation error of $F$ by low degree polynomials can be
bounded below by $\Pr[F(\phi) \neq F(\phi+\theta)]$. 
To show this, we use basic Fourier analysis to show that any low-degree polynomial with
small $L^1$ norm cannot have any large higher derivatives. This implies
that if $F$ transitions from being $0$ to being $1$ over some small
interval, that any low-degree polynomial will not be able to match it
very well in this interval.}

\paragraph{(C)SQ Lower Bounds for Real-valued Functions}
We now move to real-valued functions and sketch our CSQ and SQ lower
bounds. For CSQ lower bounds, we obtain a similar characterization.
The difference is that, in the real-valued setting, we need to find a
real-valued function $f$ whose low-degree moments vanish, and which is
close to the function we are trying to learn {\em in $L^2$ norm}. This can be
phrased as a similar LP and, applying duality, we find that the
complexity is determined by the degree needed to approximate the
function we are trying to learn in $L^2$ norm. For this particular
setting, the LP can actually be solved explicitly and the best
possible approximation function is obtained by taking the high-degree
Hermite component of $f$. This lower bound matches (up to polynomial
factors in the final error) the upper bound coming from the $L^2$ polynomial
regression algorithm. This means that we can qualitatively
characterize the complexity of agnostic learning using CSQ algorithms. 
In particular, we use this characterization to obtain new CSQ lower bounds
on agnostically learning ReLUs and sigmoids.

Our SQ lower bounds against learning
real-valued functions are somewhat more challenging, 
since the approximating function $f$ must have more than just vanishing moments.
It must have all its level-sets match low-degree moments with a standard Gaussian
(which is equivalent only for Boolean-valued functions). 
Because of this additional requirement,  we restrict our ``imitating functions'' to Boolean-valued functions. 
We can still find an LP defining $f$, however the dual gives us the 
relevant parameter of the degree needed to approximate the function we
are trying to learn in $L^1$-norm (rather than $L^2$-norm) 
for which a matching upper bound is not known. So, in this case, 
while we can still obtain significantly improved SQ lower bounds 
for agnostically learning a number of concept classes, we do not
obtain optimal results.

\subsection{Comparison to Prior Work} \label{ssec:related}

At the level of results, the most relevant prior works are the two independent works~\cite{DKZ20, GGK20},
which established the previously best SQ lower bounds for LTFs, ReLUs, and sigmoids
under the Gaussian distribution. We have already provided a technical comparison to~\cite{DKZ20}
in the previous subsection. The work~\cite{GGK20} relies on a boosting procedure that 
translates recent SQ lower bounds for (non-agnostic) learning 
one-hidden-layer neural networks~\cite{diakonikolas2020algorithms} 
to agnostically learning simple concept classes.

A useful point of technical comparison is the work~\cite{DFTWW15},
which gave an analogue of our results on agnostically learning Boolean functions 
on the Boolean hypercube. The basic statement is the same ---  that
the complexity of agnostic learning Boolean functions under a  discrete product distribution 
is characterized by the $L^1$-approximation degree --- and the duality-based 
proof techniques are similar. In particular,~\cite{DFTWW15} sets up a {\em finite} LP 
to find a function $f$ that has vanishing Fourier coefficients but is close in $L^1$-norm 
to the target function. Due to the discrete nature of the setting they consider, \cite{DFTWW15} 
avoids the functional analysis based arguments required to establish duality in our setting.  

A more significant difference with our framework is that the hard family of~\cite{DFTWW15}
embeds a copy of $f$ as a junta on a random subset of coordinates, while ours 
embeds it in a random low-dimensional subspace. This is a critical distinction
and is necessary in the Gaussian setting to obtain our tight characterization and 
the associated applications to LTFs/PTFs and intersections of halfspaces.
\new{Finally, we remark that the appendix of \cite{DFTWW15} sketches a generalization
of their results to arbitrary  product distributions (including the Gaussian distribution). 
We emphasize, however, that the lower bound obtained from their construction does 
{\em not} match the guarantee of the $L^1$-regression algorithm~\cite{KKMS:08} for the following reason:
The exponent for their lower bounds for the continuous setting have to do with the degree necessary 
to $\eps$-approximate the hard function as {\em a linear combination of $d$-juntas}. 
On the other hand, the upper bound of~\cite{KKMS:08} is related to the approximation by degree-$d$ polynomials. 
Note that degree-$d$ polynomials are always linear combinations of $d$-juntas,
and thus the approximation degree by linear combinations of juntas is 
lower than the approximation degree by polynomials. In summary, while the lower bound
of~\cite{DFTWW15} is tight for discrete product distributions, this is not true in general.}

\subsection{Preliminaries} \label{ssec:prelims}

\paragraph{Notation}
For $n \in \Z_+$, we denote $[n] \eqdef \{1, \ldots, n\}$.
We typically use small letters to denote random variables when the underlying distribution is clear from the context.
We use $\E[x]$ for the expectation of the random variable $x$
and $\pr[\mathcal{E}]$ for the probability of event $\mathcal{E}$.
We will use $\mathcal{U}(S)$ for the uniform distribution on the set $S$.
Let $\normalone$ denote the standard univariate Gaussian distribution
and $\normald{n}$ denote the standard $n$-dimensional Gaussian distribution.
We use $\phi_n$ to denote the pdf of $\normald{n}$.
Sometimes we may use the same symbol for a distribution
and its pdf, i.e., denote by $D(\bx)$ the density that the distribution
$D$ gives to the point $\bx$.

Small boldface letters are used for vectors and capital boldface letters are used for matrices. Let $\|\bx\|_2$ denote the $L^2$-norm of the vector $\bx \in \R^n$.
We use $\langle \mathbf{u}, \mathbf{v} \rangle$ for the inner product
of vectors $\mathbf{u}, \mathbf{v} \in \R^n$.
For a matrix $\vec{P} \in \R^{m \times n}$, let $\snorm{2}{ \vec{P}}$ denote
its spectral norm and $\snorm{F}{ \vec{P}}$ denote its Frobenius norm.
We use $\vec{I}_n$ to denote the $n \times n$ identity matrix.
We denote by ${\cal P}^n_d$ the class of all polynomials from $\R^n$ to $\R$
with degree at most $d$. \nnew{We sometimes use the notation $\tilde{O}(\cdot)$ (resp. $\tilde{\Omega}(\cdot)$), this is the same with $O(\cdot)$ (resp. ${\Omega}(\cdot)$), ignoring logarithmic factors, i.e., $O(d\log^k d)=\tilde{O}(d)$.}

\paragraph{Statistical Query Dimension}
To bound the complexity of SQ learning a concept class $\cal C$,
we use the SQ framework for problems over distributions~\cite{FGR+13}.

\begin{definition}[Decision Problem over Distributions] \label{def:decision}
Let $D$ be a fixed distribution and $\D$ be a distribution family.
We denote by $\mathcal{B}(\D, D)$ the decision (or hypothesis testing) problem
in which the input distribution $D'$ is promised to satisfy either
(a) $D' = D$ or (b) $D' \in \D$, and the goal
is to distinguish between the two cases.
\end{definition}

\begin{definition}[Pairwise Correlation] \label{def:pc}
The pairwise correlation of two distributions with probability density functions
$D_1, D_2 : \R^n \to \R_+$ with respect to a distribution with
density $D: \R^n \to \R_+$, where the support of $D$ contains
the supports of $D_1$ and $D_2$, is defined as
$\chi_{D}(D_1, D_2) \eqdef \int_{\R^n} D_1(\bx) D_2(\x)/D(\bx)\, \d\bx - 1$.
\end{definition}

\begin{definition} \label{def:uncor}
We say that a set of $s$ distributions $\mathcal{D} = \{D_1, \ldots , D_s \}$
over $\R^n$ is $(\gamma, \beta)$-correlated relative to a distribution $D$
if $|\chi_D(D_i, D_j)| \leq \gamma$ for all $i \neq j$,
and $|\chi_D(D_i, D_j)| \leq \beta$ for $i=j$.
\end{definition}

\begin{definition}[Statistical Query Dimension] \label{def:sq-dim}
For $\beta, \gamma > 0$, a decision problem $\mathcal{B}(\D, D)$,
where $D$ is a fixed distribution and $\D$ is a family of distributions,
let $s$ be the maximum integer such that there exists a finite set of distributions
$\mathcal{D}_D \subseteq \D$ such that
$\mathcal{D}_D$ is $(\gamma, \beta)$-correlated relative to $D$
and $|\mathcal{D}_D| \geq s.$ The {\em Statistical Query dimension}
with pairwise correlations $(\gamma, \beta)$ of $\mathcal{B}$ is defined to be $s$,
and denoted by $\mathrm{SD}(\mathcal{B},\gamma,\beta)$.
\end{definition}
\begin{lemma} \label{lem:sq-from-pairwise}
Let $\mathcal{B}(\D, D)$ be a decision problem, where $D$ is the reference distribution
and $\mathcal{D}$ is a class of distributions. For $\gamma, \beta >0$,
let $s= \mathrm{SD}(\mathcal{B}, \gamma, \beta)$.
For any $\gamma' > 0,$ any SQ algorithm for $\mathcal{B}$ requires queries of tolerance at most $\sqrt{\gamma + \gamma'}$ or makes at least
$s  \gamma' /(\beta - \gamma)$ queries.
\end{lemma}

\section{SQ Lower Bound for Boolean-Valued Concepts: Proof of Theorem~\ref{thm:main-boolean}} \label{sec:sq-bool}

The idea of our construction is to find a function $g: \R^m \to [-1, 1]$
whose low-degree moments vanish and is non-trivially close to $f$.
Our hard distribution will then embed $g$ in a random $m$-dimensional subspace.
Given this construction, we can apply Lemma~\ref{lem:sq-from-pairwise}
to prove Theorem~\ref{thm:main-boolean}.
The following result establishes the existence of such a function $g$.

\begin{proposition}\label{prop:duality1}
Let $f : \R^m \to \{\pm 1\}$ be such that for any polynomial $p:\R^m \to \R$
of degree at most $d{-1}$, it holds $\E_{\x\sim \normald{m}}[|p(\x){-f(\x)}|] \geq 2\eps$.
There exists a function $g{:\R^m\to [-1,1]}$ such that:
\begin{enumerate}
\item For any degree at most $d-1$ polynomial $P:\R^m\to \R$, we have that
      $\E_{\bx \sim \normald{m}}[P(\x)g(\x)]=0$, i.e., $g$ has zero low-degree moments, and,
\item $\E_{\bx \sim \normald{m}}[|g(\x)-f(\x)|]\leq 1-2\eps$, i.e., $g$ is non-trivially close to $f$.
\end{enumerate}
\end{proposition}
\begin{proof}
Note that such a function $g$ would be a solution to the following infinite linear program (LP): \begin{empheq}[left=(\ast)\empheqlbrace]{align}
&\E_{\bx \sim \normald{m}}[|g(\x)-f(\x)|]  \leq 1-2\eps \notag\\
& \E_{\bx \sim \normald{m}}[P(\x)g(\x)]  =0 &\forall P &\in{\cal P}_{d-1}^m \notag\\
&|g(\x)| \leq 1 &\forall \x &\in \R^m \notag
\end{empheq}
We claim that the LP $(\ast)$ is equivalent to the following LP:
\begin{empheq}[left=(\ast\ast)\empheqlbrace]{align}\label{sys:primal-boolean}
&-\E_{\bx \sim \normald{m}}[g(\x)f(\x)]+2\eps  \leq 0 & \notag\\
&\E_{\bx \sim \normald{m}}[P(\x)g(\x)]  =0 & \forall P & \in {\cal P}_{d-1}^m \notag\\
&\E_{\bx \sim \normald{m}}[g(\x)h(\x)] -\|h\|_1\leq 0 &\forall h&\in\lspace{1}(\R^m) \notag
\end{empheq}
We now show the equivalence between the two formulations. We claim
that the third constraint of  $(\ast)$ is equivalent with the third constraint of  $(\ast\ast)$.
This follows by introducing the ``dual variable'' $h:\R^m \to \R$.
The forward direction follows from H\"older's inequality
and the inverse follows from the definition of dual norms as suprema.
Finally, for the first constraints, note that since $f$ is Boolean-valued and $\|g\|_\infty\leq 1$,
we have that $\E_{\bx \sim \normald{m}}[|g(\x)-f(\x)|]=1-\E_{\bx \sim \normald{m}}[g(\x)f(\x)]$.

At this point, we would like to use ``LP duality'' to argue that $(\ast\ast)$ is feasible if and only
if its ``dual LP'' is infeasible. While such a statement turns out to be true, it requires some care to prove
since we are dealing with infinite LPs (both in number of variables and constraints). The proof requires
a version of the geometric Hahn-Banach theorem from functional analysis.\nnew{
\begin{lemma}[Informal]\label{lem:informal-duality}The LP defined by $(\ast\ast)$ is feasible
if only if there is no conical combination of the inequalities of $(\ast\ast)$
that yields the contradictory inequality $ \E_{\bx \sim \normald{m}}[g(\x)\cdot 0]+1\leq 0$.
\end{lemma}
\noindent A proof of this lemma can be found on Appendix~\ref{app:duality}.
Using Lemma~\ref{lem:informal-duality}, the LP defined by $(\ast\ast)$ is feasible if and only if the following ``dual'' LP is infeasible:}
\begin{empheq}[left=(\ast\ast')\empheqlbrace]{align}\label{sys:dual-boolean}
&\|h\|_1 -2 \lambda \, \eps<0 \notag\\
&h(\x)+P(\x)-\lambda \, f(\x) = 0  &\forall \x \in \R^m \notag \\
&{\lambda \geq 0, h \in L^1(\R^m), P\in{\cal P}_{d-1}^m}  \notag
\end{empheq}
Suppose that such a solution $(\lambda, h, P)$ exists.
    We can assume that $\lambda>0$, since otherwise the first inequality is violated.
    Moreover, by scaling the solution, we can further assume $\lambda=1$.
    Then, the second constraint becomes $h = f-P$ and the first becomes $\|f-P\|_1<2\eps$.
However, this cannot happen by the definition of the degree $d$ (since, by assumption,
there is no polynomial of degree less than $d$ such that $\|f-P\|_1< 2\eps$).
Therefore, the LP $(\ast\ast)$ is feasible, which completes our proof.
\end{proof}

Our construction will use rotated versions of the function $g$ from Proposition~\ref{prop:duality1}
to create a family of distributions that is hard to distinguish from a fixed reference distribution.
To bound the SQ dimension of this hypothesis testing problem, we will need a generalization of
Lemma~16 in~\cite{diakonikolas2020algorithms}, which bounds the correlation of 
two rotated versions of $g$. To formally state our lemma, we will need one additional piece of terminology.
If $g(\bx) = \sum_{J \in \N^m} \hat{g}(J)   H_J(\bx)$ is the Hermite expansion of $g$,
the degree-$t$ Hermite part of $g$ is the sum of the terms corresponding 
to the Hermite polynomials of degree exactly $t$. 
(For background in multilinear algebra and Hermite analysis, 
see Appendices~\ref{app:multilinear_algebra} and~\ref{app:hermite_polynomials}.)
Our main correlation lemma is the following.

\begin{lemma}[Correlation Lemma] \label{lem:cor}
Let $g: \R^m \to \R$ and $ \vec U,  \vec V \in \R^{m\times n}$ be linear maps
such that $\vec U \vec U^\top =  \vec V \vec V^\top = \vec I_m$. Then, we have that
$$\E_{\bx \sim \normald{n}}[g(\vec U\x)g(\vec V\x)] \leq
    \sum_{t=0}^\infty \|\vec U\vec V^\top\|_2^t \E_{\bx \sim \normald{m}}[(g^{[t]}(\bx))^2] \;,$$
where $g^{[t]}$ denotes the degree-$t$ Hermite part of $g$.
\end{lemma}
\begin{proof}
To simplify notation, write $g_1(\vec x) = g(\matr U \vec x)$ and 
$g_2(\vec x) = g(\matr V \vec x)$.  Moreover, we will write 
$g_1(\vec x) \sim \sum_{k=0}^{\infty} g_1^{[k]}(\vec x)$ 
and $g_2(\vec x) \sim \sum_{k=0}^{\infty} g_2^{[k]}(\vec x)$.  
Using Fact~\ref{fct:harmonic_nabla_dot}, we obtain
\begin{align} \label{eq:correlated_inner_product}
\E_{\vec x \sim \normald{n}}[g_1(\vec x) g_2(\vec x)]
 & = \sum_{k=0}^\infty \E_{\vec x \sim \normald{n}}[g_1^{[k]}(\vec x) g_2^{[k]}(\vec x)]
= \sum_{k=0}^\infty \frac{1}{k!} \dotp{\nabla^k g_1^{[k]}(\vec x)}{\nabla^k g_2^{[k]}(\vec x)} \nonumber \\
&= \sum_{k=0}^\infty \frac{1}{k!} \dotp{\nabla^k g^{[k]} (\matr U \vec x )}{\nabla^k g^{[k]}(\matr V \vec x)} \,.
\end{align}
Denote by $\mathcal{U}\subseteq \R^n$ the image of the linear map $\matr U^\top$.
Now observe that, using the chain rule, for any function 
$h(\matr U \vec x): \vec \R^n \to \R$ it holds
$\nabla h(\matr U \vec x) = \partial_{i} h(\matr U \vec x) \matr U_{ij} \in \mathcal{U}$,
where we used Einstein's summation notation for repeated indices.
Applying the above rule $k$ times, we have that
$$ \nabla^k h(\matr U \vec x) = \partial_{i_k} \ldots \partial_{i_1} h(\matr U \vec x) \matr U_{i_1j_1}
\ldots \matr U_{i_kj_k} \ \in \mathcal{U}^{\otimes k} \;.
$$
We denote $\matr R = \nabla^k g^{[k]}(\bx)$ and observe that this tensor does not depend on $\bx$.  
Moreover, denote 
$\matr M = \matr U \matr V^\top $, 
$\matr S = \nabla^k g^{[k]} (\matr U \vec x) = (\matr U^\top)^{\otimes k} \matr R \in \mathcal{U}^{\otimes k}$, 
and $\matr T = \nabla^k g^{[k]} (\matr V \vec x) = (\matr V^\top)^{\otimes k} \matr R \in \mathcal{V}^{\otimes k}$.
We have that
\begin{align*}
\dotp{\matr S}{\matr T}
= \dotp{(\matr U^\top)^{\otimes k} \matr R}
{(\matr V^\top)^{\otimes k} \matr R}
= \dotp{\matr R}{\matr M^{\otimes k} \matr R}
\leq \snorm{2}{\matr M^{\otimes k}} \snorm{2}{\matr R}^2
= k! \snorm{2}{\matr M}^k \E_{\vec x \sim \normald{n}}[(g^{[k]}(\vec x))^2] \;,
\end{align*}
where to get the last equality we used again Fact~\ref{fct:harmonic_nabla_dot}.
To finish the proof, we combine this inequality with
Equation~\eqref{eq:correlated_inner_product}.
\end{proof}

We consider high-dimensional distributions that encode a function in a subspace
and are Gaussian in the orthogonal complement. 
Using Lemma~\ref{lem:cor}, we can bound their pairwise correlations.

\begin{corollary} \label{cor:correlation-bound}
Let $d\geq 2$ and $D$ be a distribution over $\R^m$ such that the first $(d{-1})$ moments of $D$
match the corresponding moments of $\normald{m}$.
Let $G(\x)=D(\x)/\phi_m(\x)$ be the ratio of the corresponding probability density functions.
For matrices $\vec U,  \vec V \in \R^{m\times n}$ such that $\vec U \vec U^\top =  \vec V \vec V^\top = \vec I_m$,
define $D_{\vec U}$ and $D_{\vec V}$ to have probability density functions
$G(\vec U\x)\phi_n(\bx)$ and $G(\vec V\x)\phi_n(\bx)$, respectively. Then, we have that
$|\chi_{\normald{n}}(D_{\vec U},D_{\vec V})| \leq \|\vec U\vec V^\top\|_2^d \chi^2(D,\normald{m})$.
\end{corollary}

\begin{proof}
We compute
\begin{align*}
\chi_{\normald{n}}(D_{\vec U},D_{\vec V}) & = \E_{\bx \sim \normald{n}}\left[\frac{(D_{\vec U}(\bx)-\phi_n(\bx))(D_{\vec V}(\bx)-\phi_n(\bx))}{\phi^2_n(\bx)}\right]
= \E_{\bx \sim \normald{n}}[(G(\vec U\x)-1)(G(\vec V\x)-1)] \;.
\end{align*}
We then apply Lemma~\ref{lem:cor} to the function $g(\x)=G(\x)-1$. Note that the assumption that $D$ matches the first $d-1$ moments with $\normald{m}$ is equivalent to saying that $g^{[t]}=0$ for $t<d$. Thus, Lemma~\ref{lem:cor} implies that
\begin{align*}
|\chi_{\normald{n}}(D_{\vec U},D_{\vec V})| & \leq \|\vec U\vec V^\top\|_2^d \sum_{t=0}^\infty \E_{\bx \sim \normald{m}}[(g^{[t]}(\bx))^2] = \|\vec U\vec V^\top\|_2^d \E_{\bx \sim \normald{m}}[g^2(\x)] \\
                                            & \leq \|\vec U\vec V^\top\|_2^d \chi^2(D,\normald{m})\;,
\end{align*}
where the equality is Parseval's identity and in the last inequality we used the definition of $G$.
\end{proof}

Note that $D_{\vec U}$ and $D_{\vec V}$ are copies of $D$
in the subspaces defined by $\vec U$ and $\vec V$ respectively,
and independent Gaussians in the orthogonal component.

In order to create our hard family of distributions, we will need the following lemma
which states that there exist exponentially many linear operators from $\R^n$ to $\R^m$
that are nearly orthogonal. 

\begin{lemma}\label{lem:near-orth-mat}
Let $0<a,c<1/2$ and $m,n \in \Z_+$ such that $m\leq n^a$.
There exists a set $S$ of $2^{\Omega(n^c)}$ matrices in $\R^{m\times n}$
such that  every $\vec U \in S$ satisfies $\vec{U}\vec{U}^\top=\vec{I}_m$
and every pair $\vec{U},\vec{V} \in S$ with $\vec{U} \neq \vec{V}$ satisfies
$\snorm{F}{\vec{U}\vec{V}^\top} \leq O(n^{2c-1+2a})$.
\end{lemma}
\begin{proof}
Our proof relies on the following fact that there exist exponentially many nearly orthogonal unit vectors.
\begin{fact}[see, e.g., Lemma 3.7 of \cite{DKS17-sq}]\label{fct:near-orth-vec}
For any $0<c<1/2$ there exists a set $S'$ of $2^{\Omega(n^c)}$ unit vectors in $\R^n$
such that any pair $\mathbf{u}, \mathbf{v} \in S'$, with $\mathbf{u} \neq \mathbf{v}$,
satisfies $|\langle \mathbf{u},\mathbf{v} \rangle |<O(n^{c-1/2})$.
\end{fact}
\noindent Let $S'$ be the set of unit vectors that Fact~\ref{fct:near-orth-vec} constructs. 
We group them into sets of size $m$ and use the vectors of each group as rows for each matrix that we make. 
Thus, we create at least $|S'|/n^a = 2^{\Omega(n^c)}$ many matrices. Next, we ortho-normalize each matrix 
$\vec V \in S'$ using the Gram-Schmidt process, in order to get $\vec{V} \vec{V}^\top = \vec{I}_m$. 
In every row of $\vec{V}$, the Gram-Schmidt algorithm adds at most $m$ orthogonal vectors, 
each having norm $O(n^{c-1/2})$. Thus, the total correction term for each row has norm at most $\sqrt{m}O(n^{c-1/2})$. 
Putting everything together, we have that for all $\vec{U},\vec{V}$ obtained that way,
\begin{equation*}
\snorm{F}{\vec{U}\vec{V}^\top} \leq \left(m^{2} m^2 O(n^{4(c-1/2)}) \right)^{1/2}= O\left(n^{2c-1+2a} \right)\;.\qedhere
\end{equation*}
\end{proof}

\noindent We now formally define the family of distributions that we use to prove our hardness result.
\begin{definition}\label{def:hard-family}
Given a function $g:\R^m\to [-1,1]$, we define $\D_g$ to be the class of distributions over
$\R^n\times \{\pm 1\}$ of the form $(\bx,y)$ such that $\bx \sim \normald{n}$
and $\E[y| \bx=\bz] = g(\vec U\bz)$, where $\vec U\in \R^{m\times n}$ with $\vec U\vec U^\top = \vec I_m$.
\end{definition}

In the following, we show that if $g$ has low-degree moments equal to zero,
then distinguishing $\D_g$ from the distribution $(\bx,y)$ with $\bx \sim \normald{n}$, 
$y \sim \mathcal{U}(\{\pm 1\})$ is hard in the SQ model.
\begin{proposition}\label{prop:testing-lower-bound}
Let $g:\R^m\to [-1,1]$ be such that $\E_{\bx \sim \normald{m}}[g(\x)p(\x)]=0$,
for every polynomial $p: \R^m \to \R$ of degree less than $d$,
and $\D_g$ be the class of distributions from Definition~\ref{def:hard-family}.
Then, if $m\leq n^{a}$, for some constant $a<1/2$, any SQ algorithm that solves the decision problem
$\mathcal{B}(\D_g,\normald{n}\times \mathcal{U}(\{\pm 1\}))$ must either use queries of tolerance $n^{-\Omega(d)}$
or make at least $2^{n^{\Omega(1)}}$ queries.
\end{proposition}

\begin{proof}
Consider the set of matrices $S$ of Lemma~\ref{lem:near-orth-mat},
for an appropriately small value of $c>0$.
Each matrix $\vec U \in S$ is associated with a unique element of $\mathcal{D}_g$.
For every pair of distinct $\vec U,\vec V \in S$, we have that
$$\|\vec U\vec V^\top\|_2 \leq \|\vec U\vec V^\top\|_F \leq O(n^{2c-1+2a}) \leq n^{-\Omega(1)}\;,$$
where for the last inequality we chose $c$ to be a sufficiently small constant, e.g., $c=(1-2a)/4$.

Note that the distribution in $\mathcal{D}_g$ associated to a matrix $\vec U$ has probability
density $(1+g(\vec U\x))\phi_n(\x)$ when conditioned on $y=1$, and density
$(1-g(\vec U\x))\phi_n(\x)$ when conditioned on $y={-1}$.
Let $D_{\vec U}$ be the distribution associated to $\vec U$ and $D_{\vec V}$
the distribution  associated to $\vec V$. Denote by $A_{\vec U}$ the distribution $D_{\vec U}$
conditioned on the event $y=1$ and $B_{\vec U}$ the same distribution conditioned on $y={-1}$.
Similarly, let $A_{\vec V}$ and $B_{\vec V}$ denote the conditional distributions associated
with $\vec  V$. Using the definition of pairwise correlation and the fact that $y$ gets each label with equal probability,
it follows directly that
$$\chi_{\normald{n}\times \mathcal{U}( \{\pm 1\})}(D_{\vec U},D_{\vec V}) = 
\frac{1}{2}\left(  \chi_{\normald{n}}(A_{\vec U},A_{\vec V}) +\chi_{\normald{n}}(B_{\vec U},B_{\vec V}) \right)\;.$$
By Corollary~\ref{cor:correlation-bound} applied 
to $A_{\vec U},A_{\vec V}$ and $B_{\vec U},B_{\vec V}$, we obtain
$$
\chi_{\normald{n}}(A_{\vec U},A_{\vec V}) +\chi_{\normald{n}}(B_{\vec U},B_{\vec V})  
\leq  \|\vec U\vec V^\top\|_2^d \left( \chi^2(A,\normald{m}) + \chi^2(B,\normald{m}) \right)\;,
$$
where $A$ is the distribution of the random variable $\vec{U} \bx$ for $\bx \sim A_{\vec U}$ (and similarly for $B$). 
For the $\chi^2$-divergence terms, we have that
\begin{align*}
\chi^2(A,\normald{m}) 
& = \int_{\R^m} \frac{A^2(\bz)}{\phi_m(\bz)} \d \bz - 1 = \int_{\R^m} \frac{\phi_m^2(\bz)\pr^2[y=1 | \bx=\bz]}{\phi_m(\bz)\pr^2[y=1]} \d \bz - 1 \notag \\
& \leq  4 \int_{\R^m} \phi_m(\bz) \pr[y=1 | \bx=\bz] \d \bz - 1 = 4\pr[y=1] - 1 = 1\;,
\end{align*}
where we used the definition of $A$, Bayes' rule and the fact that $\pr[y=1]=1/2$.
Combining the above, we get that
$| \chi_{\normald{n}\times \mathcal{U}( \{\pm 1\})}(D_{\vec U},D_{\vec V}) | \leq  n^{-\Omega(d)}$.
This inequality implies that $\mathrm{SD}(\mathcal{B}, \gamma, \beta) = 2^{\Omega(n^c)}$,
for $\gamma = n^{-\Omega(d)}$ and $\beta = O(1)$. Using Lemma~\ref{lem:sq-from-pairwise},
with $\gamma'=\gamma$, completes the proof.
\end{proof}

\begin{proof}[Proof of Theorem~\ref{thm:main-boolean}]
Let $\mathcal{A}$ be an agnostic learner for $\mathcal{C}$.
We use $\mathcal{A}$ to solve the decision problem
$\mathcal{B}(\D_g, \normald{n}\times\mathcal{U}(\{\pm 1\}))$,
where $g:\R^m \to [-1,1]$ is the function from Proposition~\ref{prop:duality1}
and $\D_g$ the family of Definition~\ref{def:hard-family}.
Let $D'$ be the target distribution, i.e., $D' =\normald{n}\times\mathcal{U}(\{\pm 1\})$
if the null hypothesis is true or $D' \in \mathcal{D}_g$ otherwise.
We feed $\mathcal{A}$ examples drawn from $D'$
and it outputs a hypothesis $h:\R^n \to \{\pm 1\}$ such that
$\pr_{(\bx,y) \sim D'}[h(\bx) \neq y] \leq \opt + \frac{\eps}{2}$.
If $D' \in \D_g$, then for a matrix $\vec U \in \R^{m \times n}$ with $\vec{U} \vec{U}^\top=\vec{I}_m$,
we have that
$\opt \leq \pr_{(\bx,y) \sim D'}[f(\vec{U}\bx) \neq y]  = \frac{1}{2}\snorm{1}{f-g} \leq \frac{1}{2}(1- 2 \eps)$,
where in the equality we used the fact that the expectation of $y$ conditioned on $\bx$ is
$g(\bx)$ and the last inequality is due to Proposition~\ref{prop:duality1}.
Combining the above, we get that $\pr_{(\bx,y) \sim D'}[h(\bx) \neq y] \leq (1-\eps)/2$, or equivalently that
$\E_{(\bx,y) \sim D'}[h(\bx)y]\geq \eps$. On the other hand, if the labels were drawn uniformly at random,
this correlation would be exactly $0$. Therefore, we can distinguish between the two cases
by performing a final query of tolerance $\eps/2$ for the correlation of $h$ with $y$.
\end{proof}

\section{Explicit SQ Lower Bounds for Boolean Concept Classes} \label{sec:app-bool}

\subsection{LTFs and Degree-$k$ PTFs} \label{ssec:ltfs-ptfs}

Linear threshold functions (LTFs) are Boolean functions of the form
$F(\bx) = \sign ( \langle \bw, \bx \rangle + \theta)$, where $\bw \in \R^n$ and $\theta\in \R$.
A degree-$k$ PTF is any Boolean function of the form
$F(\bx) = \sign (q(\bx))$, where $q: \R^n \to \R$ is a real degree-$k$ polynomial.
In this section, we show:

\begin{theorem}[Degree Lower Bound for PTFs] \label{thm:ptfs-degree}
There exists a degree-$k$ PTF $f: \R \to \{\pm 1\}$ such that any degree-$d$ polynomial
$p : \R \to \R$ with $\snorm{1}{f-p} < \eps$ must have $d = \Omega(k^2/\eps^2)$.
\end{theorem}

Theorems~\ref{thm:main-boolean} and~\ref{thm:ptfs-degree}  imply that
any SQ algorithm that agnostically learns the class of degree-$k$ PTFs on $\R^n$ under the Gaussian distribution
must have complexity at least $n^{\Omega(k^2/\eps^2)}$.

We now elaborate on these contributions.

\paragraph{Lower Bound for LTFs}
The $L^1$-regression algorithm~\cite{KKMS:08} is known to be an agnostic learner for LTFs
under Gaussian marginals with complexity $n^{O(1/\eps^2)}$. This upper bound uses
the known fact that the $L^1$ polynomial $\eps$-approximate degree of LTFs under the Gaussian distribution
is $d = O(1/\eps^2)$ (see, e.g.,~\cite{DKN10}). This upper bound is tight. Specifically,
known results in approximation theory (see Appendix~\ref{ssec:ganz-lb}) imply that,
any polynomial that $\eps$-approximates the function $\sign(t)$ in $L^1$-norm,
under the standard Gaussian distribution, requires degree $\Omega(1/\eps^2)$.
Given this structural result, an application of Theorem~\ref{thm:main-boolean},
for $m=1$ and $f(t) = \sign(t)$ gives the tight SQ lower bound of $n^{\Omega(1/\eps^2)}$.
This bound improves on the best previous bound of $n^{\Omega(1/\eps)}$~\cite{GGK20, DKZ20}.
Importantly, our approach is much simpler and generalizes to any concept class satisfying the mild
assumptions of Theorem~\ref{thm:main-boolean}.

\paragraph{Lower Bound for Degree-$k$ PTFs}
The $L^1$-regression algorithm is known to be an agnostic learner for degree-$k$ PTFs
under Gaussian marginals with complexity $n^{O(k^2/\eps^4)}$. This upper bound uses
the known upper bound of $O(k \sqrt{\eps})$ on the Gaussian noise sensitivity of degree-$k$ PTFs~\cite{Kane:10},
which implies an upper bound of $O(k^2/\eps^2)$ on the $L^2$ polynomial $\eps$-approximate degree, and therefore
an upper bound of  $O(k^2/\eps^4)$ on the $L^1$ polynomial $\eps$-approximate degree. This degree upper bound
is not known to be optimal (in fact, it is provably sub-optimal for $k=1$) and it is a plausible conjecture
that the right answer is $\Theta(k^2/\eps^2)$. Here we prove a lower bound of $\Omega(k^2/\eps^2)$,
which applies even for the univariate case.

\begin{proposition} \label{prop:ptfs-degree}
There exists a $({k+1})$-piecewise-constant function $f : \R \to \{0,1\}$ such that
any degree-$d$ polynomial $p : \R \to \R$ that satisfies $\snorm{1}{f-p} < \eps$ must have
$d = \Omega(k^2/\eps^2)$.
\end{proposition}

An application of Theorem~\ref{thm:main-boolean}, for $m=1$ and $f(t)$ being
the piecewise constant function of Proposition~\ref{prop:ptfs-degree}, implies an
SQ lower bound of $n^{\Omega(k^2/\eps^2)}$.

Before we provide the formal proof, we sketch the proof of Proposition~\ref{prop:ptfs-degree}.
The hard function $f$ consists of $k/2$ intervals with the same carefully chosen length; 
we split each interval in half and we let $f=0$ in the first half, and $f=1$ in the second half.
We construct a distribution $D$ that puts almost all of its mass in the first half of each interval, 
matches the first $d$ moments with the standard Gaussian, and $D(x)\leq 2 \phi(x)$ for all $\x\in \R$.
Then, by construction $\E_{x \sim \normalone}[f(x)]$ is much larger than the same expectation under $D$.
We show that, in fact, this difference bounds from below the error of any degree-$d$
polynomial approximation to the function $f$.

The main technical lemma we establish in this context is the following:
\begin{lemma} \label{lem:matching-moments-distr}
There exists a univariate distribution $D$ that (i) matches its first $d$ moments with $\normalone$,
(ii) the pdf of $D$ is at most $2$ times the pdf of $\normalone$ pointwise in $\R$,
and (iii) for some $\alpha = \Theta(1/\sqrt{d})$ it holds that
$\pr[(X \mod a) \in (a/2,a)] = 2^{-\Omega(d)}$.
\end{lemma}

\new{We defer the proof of Lemma~\ref{lem:matching-moments-distr} to Section~\ref{ssec:proof-matching}
and show how it implies Proposition~\ref{prop:ptfs-degree} below.}

\begin{proof}[Proof of Proposition~\ref{prop:ptfs-degree}]
We can assume that $k$ is even.
Let $f$ be $1$ on the $k/2$ intervals $(ia+a/2,(i+1)a)$, for $i=0,\ldots, k/2-1$,
and zero elsewhere. Denote by $D$ the distribution of Lemma~\ref{lem:matching-moments-distr}.
From property (iii), we have that $\E_{x \sim D}[f(x)]= 2^{-\Omega(d)}k$.
On the other hand, assuming that $k=O(\sqrt{d})$, we have that
$\E_{x \sim \normalone}[f(x)] = \Omega(k/\sqrt{d})$. This is because the regions where $f$
is $1$ are contained in the interval $[0,\Theta(k/\sqrt{d})]\subseteq [0,O(1)]$,
where the pdf of the standard Gaussian is bounded below by some constant.

Let $D(x)$ and $\phi(x)$ denote the density on point $x$ of the distribution $D$ and $\normalone$ respectively.
For every polynomial $p:\R \to \R$ of degree at most $d$, it holds
\begin{align*}
\E_{x \sim \normalone}[f(x)] - \E_{x \sim D}[f(x)] & = \E_{x \sim \normalone}\left[f(x)\left(1-\frac{D(x)}{\phi(x)}\right)\right]
= \E_{x \sim \normalone}\left[(f(x)-p(x))\left(1-\frac{D(x)}{\phi(x)}\right)\right]\\ 
& \leq \E_{x \sim \normalone}[|f(x)-p(x)|] \;,
\end{align*}
where the second equality follows from the fact that $D$ matches its first $d$ moments with $\normalone$,
and in the last inequality we used that $0\leq D(x) \leq 2 \phi(x)$ for all $x \in \R$.
Thus, if $f$ could be $L^1$-approximated to error $\eps$ by a degree-$d$ polynomial, then
$\E_{x \sim \normalone}[f(x)] - \E_{x \sim D}[f(x)]$ would be at most $\eps$.
But we already showed that this is $\Omega(k/\sqrt{d})$, which implies that $d = \Omega(k^2/\eps^2)$.
\end{proof}

\subsection{Proof of Lemma~\ref{lem:matching-moments-distr}}\label{ssec:proof-matching}

First, we need the following lemma.
\begin{lemma}\label{lem:d-wise-ind-gaussians}
There is a $d$-wise independent family of $t=O(d)$ standard Gaussians $X_1,X_2,\ldots , X_t$
such that $\left(\sum_{i=1}^t X_i \right) \mod 1 \in [0,1/2]$ with probability $1-2^{-\Omega(d)}$.
Furthermore, such a distribution can be obtained by rejection sampling a set of independent standard Gaussians,
where a sample is rejected with probability $1/2$.
\end{lemma}
\begin{proof}
The standard Gaussian distribution can be decomposed into a uniform component and a remaining term. 
That is, $\normalone = c \, \mathcal{U}([0,1]) + (1-c) E$, where $\mathcal{U}([0,1])$ is the uniform distribution in $[0,1]$, 
$E$ is another distribution, and $c>0$ is a constant. Let $t\in \N$ such that $t>d/c$. 
We generate this $d$-wise independent family $X_1,\ldots, X_t$ as follows.

First, we sample $Y_1,\ldots, Y_t$ independent standard Gaussians, 
writing each $Y_i$  either as a sample from $\mathcal{U}([0,1])$ or a sample from $E$. 
Then, two complementary cases are considered.
\begin{itemize}
\item[] \textbf{Case 1}. The number of $Y_i$'s that came from $\mathcal{U}([0,1])$ is at most $d$. 
In this case, the sample is rejected with probability $1/2$.
\item[] \textbf{Case 2}. Otherwise, the sample is rejected if and only if $(\sum_{i=1}^t Y_i) \mod 1 \in (1/2,1]$.
\end{itemize}
Let $X_1,\ldots, X_t$ be the output of this rejection sampling procedure. 
The probability that the sample is generated by the first case of the algorithm is exponentially small. 
To see this, define $Z_i \in \{0,1\}$  to be one if and only $Y_i$ is drawn from $\mathcal{U}([0,1])$. 
If $C_1$ denotes the event of being in Case 1, then by standard Chernoff bounds we have that
\begin{align*}
\pr[C_1] 
& = \pr\left[\sum_{i=1}^t Z_i \leq d \right] 
= \pr\left[\sum_{i=1}^t Z_i \leq  \E\left[\sum_{i=1}^t Z_i \right] \left(1- \left( 1 -\frac{d}{t c}\right) \right)\right] \\
&\leq \exp \left( - \frac{(1-d/(tc))^2 tc}{2} \right) = 2^{-\Omega(d)} \;,
\end{align*}
where we used that $t>d/c$. Therefore, the probability that 
$(\sum_{i=1}^t X_i) \mod 1 \in [0,1/2]$ is $1-2^{-\Omega(d)}$.

Moreover, the probability of accepting the sample is exactly $1/2$ independently of the $Y_i$'s. 
To see this, let $C_1,C_2=\overline{C_1}$ be the events of Case 1 and Case 2 being true respectively, 
and $A$ be the event of accepting the sample. For Case 1, we have $\pr[A|C_1]=1/2$. 
In Case 2, we know that at least one element is drawn from $\mathcal{U}([0,1])$, 
which means that the $(\sum_{i=1}^t X_i) \mod 1$ is going to be uniform in $[0,1]$. 
Thus, $\pr[A|C_2]=1/2$. Therefore, 
$\pr[C_1|A]=\pr[A|C_1]\pr[C_1]/\pr[A] = \pr[C_1]$ and 
$\pr[C_2|A]=\pr[A|C_2]\pr[C_2]/\pr[A] = \pr[C_2]$, 
i.e., accepting is independent of $C_1,C_2$, and thus independent of the sample itself. 
This means that the output $X_1,\ldots, X_t$ remains Gaussian.

For the $d$-wise independence of the variables $X_1,\ldots, X_t$, 
let $\mathcal{I}$ be an arbitrary set of at most $d$ indices from $\{1,\ldots,t\}$. 
We claim that $\{X_i\}_{i \in \mathcal{I}}$ are independent. 
Case 1 is trivial, since we accept independently of the values of the $Y_i$'s. 
For Case 2, note that in that case there are more than $d$ $Y_i$'s drawn from $\mathcal{U}([0,1])$. 
This means that there exists one $j \not\in \mathcal{I}$ such that $Y_j$ is uniform 
and forces the $(\sum_{i=1}^t X_i)$ to be uniform in $[0,1]$. 
Thus, the event $(\sum_{i=1}^t X_i) \in [0,1/2]$ is independent of $\{Y_i\}_{i \in \mathcal{I}}$, 
and therefore $\{X_i\}_{i \in \mathcal{I}}$ is a set of independent random variables.
\end{proof}

\begin{proof}[Proof of Lemma~\ref{lem:matching-moments-distr}]
Consider the random variable $X = \sum_{i=1}^t X_i/ \sqrt{t}$ 
for the $X_i$'s of Lemma~\ref{lem:d-wise-ind-gaussians}. 
For (i), note that the $d$-th moment involves the expectation of at most $d$ of the $X_i$'s, 
which are independent. Note that (ii) holds because the distribution of $X$ 
puts almost all of its mass on half of the real line, and (iii) follows from our scaling of $1/\sqrt{t}$.
\end{proof}

\subsection{Intersections of Halfspaces: Degree Lower Bound via Gaussian Noise Sensitivity} \label{ssec:inters}

An intersection of $k$ halfspaces on $\R^n$ is any function $f:\R^n \to \{\pm1\}$ such that there exist $k$
LTFs $h_i:\R^n \to \{\pm1\}$, $i \in [k]$, such that $f(\bx) = 1$ if and only if $h_i(\bx)=1$ for all $i \in [k]$.

The $L^1$-regression algorithm~\cite{KKMS:08} is known to be an agnostic learner for
intersection of $k$ halfspaces on $\R^n$ under Gaussian marginals with complexity $n^{O((\log k)/\eps^4)}$.
This upper bound uses the known tight upper bound of $O(\sqrt{\eps\log k })$
on the Gaussian noise sensitivity of this concept class~\cite{KOS:08},
which implies an upper bound of  $O(\log k/\eps^4)$ on the $L^1$ polynomial $\eps$-degree.
This degree upper bound is not known to be optimal (in fact, it is provably suboptimal for $k=1$)
and it is a plausible conjecture that the right answer is $\Theta(\sqrt{\log k}/\eps^2)$.
Here we prove a lower bound of $\tilde{\Omega}(\sqrt{\log k}/\eps)$, which applies even for $k$-dimensional
functions.

\begin{theorem}[Degree Lower Bound for Intersections of Halfspaces] \label{thm:inters-degree}
There exists an intersection of $k$ halfspaces $f$ on $\R^k$ such that the following holds:
Any degree-$d$ polynomial $p : \R^k \to \R$ that satisfies $\snorm{1}{f-p} < \eps$
must have $d=\tilde{\Omega}(\sqrt{\log k}/\eps)$.
\end{theorem}

Theorem~\ref{thm:inters-degree} combined with Theorem~\ref{thm:main-boolean},
applied for $m=k$ and $f$ being the function from Theorem~\ref{thm:inters-degree},
implies that any SQ algorithm that agnostically learns
intersections of $k$ halfspaces on $\R^n$ under the Gaussian distribution
must have complexity at least $n^{\tilde{\Omega}(\sqrt{\log k}/\eps)}$.

To prove Theorem~\ref{thm:inters-degree}, we make essential use
of our structural result, Theorem~\ref{thm:bound-l1-by-gns}, combined with the following
tight lower bound on the Gaussian noise sensitivity of a well-chosen family 
of intersection of halfspaces (see Appendix~\ref{ssec:gns-lb-inters} for the proof).

\begin{lemma}\label{lem:gns-lb-inters}
There exists an intersection of $k$ halfspaces on $\R^k$, $f: \R^k \to \{\pm 1\}$, such that
$\gns_{\eps} (f) = \Omega(\sqrt{\eps \log k})$.
\end{lemma}

\subsection{Proof of Theorem~\ref{thm:bound-l1-by-gns}} \label{ssec:structural-gns}

We require the following proposition.

\begin{proposition}\label{prop:bound-l1}
Let $p(\theta)$ be a degree-$d$ polynomial on the circle, 
i.e., a degree at most $d$ polynomial in $\sin\theta$ and $\cos\theta$, 
and let $B(\theta)$ be a Boolean-valued function that is periodic modulo $2\pi$. 
Then, for $t$ being a sufficiently small multiple of $\log\,d/d$, it holds
\begin{equation*}
\frac{1}{2\pi}\int_{0}^{2\pi} |p(\theta)-B(\theta)|\d\theta =
\tilde \Omega(1/\log d)\pr_{\phi \sim \mathcal{U}([0,2\pi])}[B(\phi-t)\neq B(\phi+t)]\;.
\end{equation*}
\end{proposition}

\begin{proof}
We can assume that $\frac{1}{2\pi}\int_{0}^{2\pi} |p(\theta)|\d\theta$ is at most $2$, 
since otherwise the $ \frac{1}{2\pi}\int_{0}^{2\pi} |p(\theta)-B(\theta)|\d\theta $ is at least $1$. 
Let $k$ be an odd integer proportional to $\log d$.
We start with the following technical claim.
\begin{claim}\label{clm:bounded-derivative} For any $\theta\in [0,2\pi]$, it holds $|p^{(k)}(\theta)| = O(d)^k$.
\end{claim}

\begin{proof}
Using $\cos\theta=\left(e^{i\theta}+e^{-i\theta}\right)/2$ and $\sin\theta=\left(e^{i\theta}-e^{-i\theta}\right)/2$, 
we write $p(\theta) = \sum_{n=-\infty}^\infty a_n e^{n i \theta}$, for some coefficients $a_n$, where
$ a_n = \frac{1}{2\pi}\int_0^{2\pi} p(\phi)e^{-n i \theta} \d\phi$.
Since $p$ has degree at most $d$, it holds that $a_n=0$, for all $n>d$ and $n<-d$.  Therefore, we have that
$p(\theta) = \sum_{n=-d}^d \frac{1}{2\pi} \int_0^{2\pi} p(\phi)e^{n i (\theta-\phi)} \d\phi$.
Taking the $k$-{th} derivative (using Leibniz's rule) gives
\[
    p^{(k)}(\theta) = \sum_{n=-d}^d \frac{1}{2\pi} \int_0^{2\pi} p(\phi)(n i)^k e^{n i(\theta-\phi)} \d\phi \;.
\]
This implies that
\[
    |p^{(k)}(\theta)| \leq \sum_{n=-d}^d \frac{1}{2\pi} \int_0^{2\pi} |p(\phi)|n^k \d\phi \leq 2 \sum_{n=-d}^d n^k = O(d^{k+1}) \,.
\]
Moreover, $k$ is proportional to $\log d$, thus $|p^{(k)}(\theta)| = O(d)^k$, for all $\theta\in [0,2\pi]$.
\end{proof}

We next pick $t$ to be a small multiple of $\log\, d/d$ and $\phi\in[0,2\pi]$.
Let $z_m = t\cos(\pi m/k)+\phi$, for $m=0,1,\ldots,k$, and
let $q(z)=\sum_{j=0}^k c_j z^j$ be the unique degree-$k$ polynomial such that
$q(z_m)=p(z_m)$, for $m=0,1,\ldots,k$.
Observe that $q-p$ has $k+1$ zeroes.
Therefore, iterating Rolle's theorem, we obtain
that there is a point $\phi-t\leq z \leq \phi+t$ such that $p^{(k)}(z)=q^{(k)}(z)$,
and thus $|q^{(k)}(z)| = O(d)^k$, or equivalently $c_k=2^kO(d/k)^k$.

Let $R(\theta)=q(t\cos \theta +\phi)$.
For some constants $b_n$ (which depend on $t$ and $\phi$),
we have that $R(\theta) = \sum_{n=-k}^k b_n e^{n i \theta}$.
Since $R(\theta)$ is an even function, its Fourier coefficients are real numbers.
The following claim provides an upper bound on the coefficient $b_k$.

\begin{claim}\label{clm:bounded-coef}
It holds that  $|b_k| \leq 1/(4k)$.
\end{claim}

\begin{proof}
Note that $b_k = (1/2\pi)\int_0^{2\pi} R( \theta) e^{-k i \theta} \d \theta$. Using the orthogonality
of the trigonometric polynomials, only terms containing $\cos (k \theta)$ are non-zero.
Moreover, $\cos^k\theta = \sum_{j=0}^k u_j \cos(j \theta)$ with $u_k = 2^{-k+1}$, 
which can be verified using the identity $\cos \theta = (e^{i\theta} + e^{-i\theta})/2$. 
Therefore, we have that
\begin{equation*}
b_k = \frac{1}{2 \pi}\int_0^{2\pi}R(\theta) e^{-k i \theta} \d \theta =  
\frac{1}{2 \pi}\int_0^{2\pi} c_k u_k t^k \cos(k \theta) e^{-k i \theta} \d \theta 
= c_k u_k \frac{t^k}{2 \pi}  \pi = \left( \frac{t}{2}\right)^k c_k \;,
\end{equation*}
where we used that $R(\theta)=\sum_{j=0}^k c_j (t \cos\theta + \phi)^j$. 
Since $c_k=2^kO(d/k)^k$, we have that $b_k = O(td/k)^k$; 
this is at most $1/(4k)$, if $t$ is a small enough multiple of $\log\, d/d$.
\end{proof}

On the other hand, by doing a filtering using the $(2k)$-th roots of unity,
we get  that $\sum_{m=0}^{2k-1} R(2\pi m/(2k))=2k b_k$,
and this is equivalent to $\sum_{m=-k+1}^k q(t\cos (\pi m/k)+\phi) (-1)^m=2kb_k$.
Therefore,
\begin{align*}
b_k & = \frac{1}{2k}\sum_{m=-k+1}^k q(t\cos (\pi m/k)+\phi) (-1)^m = \frac{1}{2k}\sum_{m=-k+1}^k p(z_{|m|}) (-1)^m\\ 
& = \frac{1}{2k} \Big(\sum_{m=-k+1}^k (p(z_{|m|})-B(z_{|m|})) (-1)^m+\sum_{m=-k+1}^k B(z_{|m|}) (-1)^m+(B(\phi+t)-B(\phi-t))\Big)\;.
\end{align*}
Since $k-1$ is even and $B$ is Boolean,
$2\sum_{m=1}^{k-1} B(z_{m}) (-1)^m$ is a multiple of $4$.
If $B(\phi+t)\neq B(\phi-t)$, the reverse triangle inequality gives
$\left| B(\phi+t)-B(\phi-t)+2\sum_{m=1}^{k-1} B(z_{m}) (-1)^m \right| \geq 2$.
Therefore, in this case, we have that
$\frac{1}{4k} > |b_k| \geq \frac{1}{2k}\left(2 - \sum_{m=-k+1}^k |p(z_{|m|})-B(z_{|m|})| \right)$,
or in other words,
$$ \sum_{m=-k+1}^k |p(z_{|m|})-B(z_{|m|})| \geq \mathds{1}\{B(\phi+t)\neq B(\phi-t)\} \;.$$
Integrating this over $\phi$ from $0$ to $2\pi$ gives
$$\int_0^{2\pi} |p(\theta)-B(\theta)|d\theta \geq \frac{\pi}{k}\pr_{\phi \sim \mathcal{U}([0,2\pi])}[B(\phi-t)\neq B(\phi+t)]\;.$$
The result follows from our assumption that $k$ is proportional to $\log d$.
\end{proof}

Using Proposition~\ref{prop:bound-l1}, we can prove the main theorem of this section.

\begin{proof}[Proof of Theorem~\ref{thm:bound-l1-by-gns}]
The latter statement follows from the fact that
$\E_{\bx \sim \normald{n}}[|f(\bx)-p(\bx)|] \geq \E_{\bx \sim \normald{n}}[|f(\bx)-\sgn(p(\bx))|/2]$.
On the other hand, we can write
\begin{align*}
& \gns_\eps(f)-\gns_\eps(\sgn(p)) 
= \pr_{(\bx,\by)\sim \normald{d}^\eps}[f(\bx)\neq f(\by)]-\pr_{(\bx,\by)\sim \normald{d}^\eps}[\sgn(p(\bx))\neq \sgn(p(\by))]\\                                                                                                     & \leq \pr_{\bx \sim \normald{n}}[f(\bx)\neq \sgn(p(\bx))] + 
\pr_{\bx \sim \normald{n}}[f(\by)\neq \sgn(p(\by))] = 2\E_{\bx \sim \normald{n}}[|f(\bx)-\sgn(p(\bx))| \;.
\end{align*}
Combining these, we find that
$\E_{\bx \sim \normald{n}}[|f(\bx)-p(\bx)|] \geq \gns_\eps(f)/4 - \gns_\eps(\sgn(p))/4$.
The result then follows from noting that $\sgn(p)$ is a degree-$d$ PTF,
and therefore by~\cite{Kane:10} it holds that $\gns_\eps(\sgn(p))=O(d\sqrt{\eps})$.

For the first statement, let $\by$ and $\bz$ be independent Gaussians and let
$\bx(\phi) = \cos\phi\,\by+\sin\phi\,\bz$. Let $a$ be a sufficiently small multiple of $\log\, d/d$.
For any $\phi\in[0,2\pi]$, $\bx(\phi-a)$ and $\bx(\phi+a)$ are
$(1{-\delta})$-correlated Gaussian random variables, where $\delta=\Theta(\log d/d)^2$.
We have that
\begin{align*}
\E_{\bx \sim \normald{n}}[|f(\bx)-p(\bx)|]
 & = \E_{\phi\in \mathcal{U}([0,2\pi])}\left[\E_{\by,\bz \sim \normald{n}}[|f(\bx(\phi)) -p(\bx(\phi))|]\right]
= \E_{\by,\bz \sim \normald{n}}\left[\E_{\phi\in \mathcal{U}([0,2\pi])}[|f(\bx(\phi)) -p(\bx(\phi))|]\right]                      \\
 & \geq \Omega(1/\log(d))\E_{\by,\bz \sim \normald{n}}[\pr_{\phi\in\mathcal{U}([0,2\pi])}[f(\bx(\phi-a))\neq f(\bx(\phi+a))]] \;,
\end{align*}
where in the inequality we used Proposition~\ref{prop:bound-l1}.
Moreover, using Fubini's theorem, we have
\begin{align*}
& \E_{\bx \sim \normald{n}}[|f(\bx)-p(\bx)|]
\geq \Omega(1/\log(d))\E_{\phi\in\mathcal{U}([0,2\pi])}\left[\pr_{\by, \bz \sim \normald{n}}[f(\bx(\phi-a))\neq f(\bx(\phi+a))]\right] \\
 & = \Omega(1/\log(d))\E_{\phi\in\mathcal{U}([0,2\pi])}[\gns_\delta(f)] = \Omega(1/\log(d))\gns_\delta(f) 
 = \Omega(1/\log(d))\gns_{(\log(d)/d)^2}(f)\,.\quad \qedhere
\end{align*}
\end{proof}

 \section{Lower Bound for Real-Valued Functions} \label{sec:l2}

\subsection{CSQ Lower Bound: Proof of Theorem~\ref{them:real-val-lb-csq}}

To prove our CSQ lower bound, we need to find a hard function $g:\R^m \to \R$ 
that is uncorrelated with low-degree polynomials and, at the same time, is close to $f$ in the $L^2$-sense. 
Instead of using duality to establish the existence of such a function $g$,
 we let $g$ be the orthogonal component of the truncated Hermite expansion of $f$.

\begin{proof}[Proof of Theorem~\ref{them:real-val-lb-csq}]
Let an algorithm $\mathcal{A}$ that agnostically learns $\mathcal{C}$ up to $L^2$-error $\eps$.
Let $g(\x)=f(\x)-\sum_{i=0}^{d-1} f^{[i]}(\x)$, i.e., $g$ is the same as the function $f$ without the low-degree moments up to $d-1$. Note that $\snorm{2}{g}\geq\eps$. Let $C=2/(\eps \snorm{2}{g})$ 
and let $S$ be the set of nearly orthogonal matrices of Lemma~\ref{lem:near-orth-mat}. 
Consider the class $\mathcal{C}_g$ that consists of all functions from $\R^n$ to $\R$ of the form 
$G_{\vec V}(\bx) = C g(\vec{V} \bx)$, for any matrix $\vec{V} \in S$. 
Every $G_{\vec V} \in \mathcal{C}_g$ is orthogonal to all polynomials of degree less than $d$, 
and also $\snorm{2}{G_{\vec V}} = 2/ \eps$.
We feed $\mathcal{A}$ with samples $(\bx,G_{\vec V}(\bx))$, 
where $\bx \sim \normald{n}$, $\vec{V} \in S$. 
Let $\eps'>0$ be the accuracy parameter used with $\mathcal{A}$. 
Then, $\mathcal{A}$ returns a hypothesis $h$ satisfying
\begin{align} \label{eq:guaranteecsq}
\sqrt{\E_{\bx \sim \normald{n}}[(h(\bx)-G_{\vec V}(\bx))^2] }\leq \opt + \eps' \;.
\end{align}
For our choice of $C$, the optimal error becomes
\begin{align*}
\opt &\leq \sqrt{\E_{\bx \sim \normald{n}}[(f(\bx)-G_{\vec V}(\bx))^2] } = 
\sqrt{1 + C^2 \snorm{2}{g}^2 - 2 C \E_{\bx \sim \normald{m}}[f(\bx)g(\bx)]} \\
&\leq \sqrt{1+\frac{4}{\eps^2}-\frac{2\snorm{2}{g}}{\eps}}
\leq \sqrt{\frac{4}{\eps^2}- 1}
\leq \frac{2}{\eps}\sqrt{1-\frac{\eps^2}{4}} 
\leq \frac{2}{\eps} - \frac{\eps}{4}
\;,
\end{align*}
where in the second inequality we used that 
$2\E_{\bx \sim \normald{m}}[f(\bx)g(\bx)] \geq \snorm{2}{g}^2$. 
By choosing $\eps' = \eps/8$, Equation~\eqref{eq:guaranteecsq} becomes $\snorm{2}{h-G_{\vec{V}}} \leq  2/\eps - \eps/8$.

It remains to bound from above the pairwise correlation of the class $\mathcal{C}_g$. 
For any two different $\vec{U},\vec{V} \in S$, we have that
\begin{align*}
\E_{\bx \sim \normald{n}}[G_{\vec{U}}(\bx) G_{\vec{V}}(\bx)] 
& \leq C^2 \sum_{t=0}^\infty \snorm{2}{\vec{U} \vec{V}^\top}^t \E_{\bx \sim \normald{m}}[(g^{[t]}(\bx) )^2]
\leq C^2  \snorm{2}{\vec{U} \vec{V}^\top}^d \sum_{t =d}^{\infty}  \E_{\bx \sim \normald{m}}[(g^{[t]}(\bx) )^2]\\                                                                      & \leq 4\eps^{-2}\snorm{F}{\vec{U} \vec{V}^\top}^d
\leq \eps^{-2} n^{-\Omega(d)} \leq n^{-\Omega(d)} \;,
\end{align*}
where in the first inequality we used Lemma~\ref{lem:cor}, 
in the second inequality we used the fact that $g$ is uncorrelated with all polynomials of degree less than $d$, 
the third inequality follows from Parseval's identity and the fact that $\snorm{2}{g} C =2/\eps$, 
the next one follows from Lemma~\ref{lem:near-orth-mat}, 
and the last one from our assumption $\eps > n^{-c}$ for an appropriate constant $c$. 
As a note, we extend our class $\mathcal{C}_g$ to include the identically zero function, 
which does not increase the pairwise correlations. Using Lemma~\ref{lem:convert-SD-to-SDA} with $\gamma'=\gamma$, 
we have that $\mathrm{CSDA}_{\normald{n}}(\mathcal{C}_g, 2 \gamma) = 2^{n^{\Omega(1)}}$ 
for $\gamma =n^{-\Omega(d)}$. An application of Lemma~\ref{lem:gen-csq-lb} with $\eta=2/\eps$ concludes the proof.
\end{proof}

\subsection{SQ Lower Bound: Proof of Theorem~\ref{them:real-val-lb-sq}}

To prove lower bounds for the general SQ model, 
we require our hard function $g$ to be pointwise bounded. 
This allows us to define a learning problem with Boolean labels, 
for which we have SQ lower bounds ready to be used. 
Because of our $L^\infty$ constraint on $g$, the resulting lower bound 
is expressed in terms of the degrees of polynomials that approximate $f$ in $L^1$ rather than $L^2$ sense.

\noindent Our duality argument will now use the pair of dual norms $L^1,L^\infty$.

\begin{proposition} \label{prop:duality-real-sq}
Let $f \in L^2(\R^m)$ be such that for any degree at most $d-1$ polynomial $p:\R^m \to \R$, it holds $\snorm{1}{f-p} \geq  \eps $. Then, there exists a function $g:\R^m\to [-1,1]$ such that:
\begin{enumerate}
\item $\E_{x \sim \normald{m}}[f(\bx)g(\bx)] \geq \eps  $, and,
\item $\E_{\bx \sim \normald{m}}[P(\x)g(\x)]=0$, for any polynomial $P:\R^m \to \R$ with degree less than $d$.
\end{enumerate}
\end{proposition}

\begin{proof}
\noindent The function $g$ is a solution to the infinite system:
\begin{empheq}[left=(\ast)\empheqlbrace]{align}
&\E_{x \sim \normald{m}}[f(\bx)g(\bx)] \geq \eps  &\notag
\\& \E_{\bx \sim \normald{m}}[P(\x)g(\x)]  =0 &\forall P\in{\cal P}_{d-1}^m \notag
\\&\|g\|_\infty \leq 1 \notag
\end{empheq}
\noindent This is equivalent to the following LP:
\begin{empheq}[left=(\ast\ast)\empheqlbrace]{align}
&- \E_{x \sim \normald{m}}[f(\bx)g(\bx)] + \eps  \leq 0 &\notag
\\& \E_{\bx \sim \normald{m}}[P(\x)g(\x)]  =0 &\forall P&\in{\cal P}_{d-1}^m \notag
\\&\E_{x \sim \normald{m}}[g(\bx) h(\bx)]\leq \|h\|_1 &\forall h &\in L^1(\R^m)\notag
\end{empheq}

\noindent From Corollary~\ref{lem:duality}, the above LP is feasible unless the following is infeasible:
\begin{empheq}[left=(\ast\ast')\empheqlbrace]{align}\label{sys:dual-csq}
&\|h\|_1 - \lambda\eps  <0 \notag
\\ &h(\x)+P(\x)-\lambda f(\x) = 0, &\forall \bx \in \R^m \notag
\\ &\lambda\geq 0, h \in L^1(\R^m), P\in{\cal P}_{d-1}^m \notag
\end{empheq}
Let $(h,P,\lambda)$ be a solution to $(\ast\ast')$. Note that we can assume that $\lambda=1$ since all constraints are homogeneous. Then, the constraints become $h = f-P$ and
\begin{equation*}
\snorm{1}{f-P} < \eps \;,
\end{equation*}
which is a contradiction. Therefore, the original system $(\ast)$ is feasible.
\end{proof}

We conclude with the proof of the main theorem for this section.
\begin{proof}[Proof of Theorem~\ref{them:real-val-lb-sq}]
Suppose that we have such an agnostic learner $\mathcal{A}$. Let $g:\R^m \to [-1,1]$ be the function of Proposition~\ref{prop:duality-real-sq}, for a parameter $\eps'>0$ to be specified.
Let $\mathcal{D}_g$ be the family of distributions over $\R^n \times \{\pm 1\}$ from Definition~\ref{def:hard-family}.
We use $\mathcal{A}$ to solve the problem of distinguishing between a distribution from $\mathcal{D}_g$ and the distribution where the labels are drawn uniformly at random.
That is, we convert $\mathcal{A}$ into an algorithm for $\mathcal{B}(\mathcal{D}_g, \normald{n} \times \mathcal{U}(\{ \pm 1 \}))$, and the hardness result will follow from the hardness of that decision problem, as established by Proposition~\ref{prop:testing-lower-bound}.

Let $D'$ be a distribution that is either $D'=\normald{n}\times \mathcal{U}(\{ \pm 1 \}))$ or $D' \in \mathcal{D}_g$. 
We feed $\mathcal{A}$ a set of i.i.d.\ samples of the form $(\bx,Cy)$, 
where $(\bx,y)\sim D'$ and $C = 1/\E_{\bx \sim \normald{m}}[f(\bx)g(\bx)]$. 
Let $\eps'>0$ be the accuracy parameter used when running $\mathcal{A}$ and $h$ be the returned hypothesis. 
We have that
\begin{align} \label{eq:guarantee}
\sqrt{\E_{(\bx,y) \sim D'}[(h(\bx)-Cy)^2]} \leq \opt + \eps' \;.
\end{align}
If $D' \in \mathcal{D}_g$, for the optimal error we have that
\begin{align*}
\opt \leq  \sqrt{1 + C^2 - 2C \E_{\bx \sim \normald{m}}[f(\bx)g(\bx)] }
\leq \sqrt{C^2 - 1} =C \sqrt{1-1/C^2}
\leq C - 1/(2C) \;.
\end{align*}
If we choose $\eps' = 1/(4C)$, Equation~\eqref{eq:guarantee} becomes 
$\sqrt{\E_{(\bx,y) \sim D'}[(h(\bx)-Cy)^2]} \leq C - 1/(4C)$. 
On the other hand, we can write 
$\sqrt{\E_{(\bx,y) \sim D'}[(h(\bx)-Cy)^2]} \geq \sqrt{C^2 - 2 C \E_{\bx \sim \normald{m}}[h(\bx)y]}$. 
Combining these two, we obtain
\begin{align*}
2 C \E_{\bx \sim \normald{m}}[h(\bx)y] \geq C^2 - (C - 1/(4C))^2 \geq 1/3 \;,
\end{align*}
which gives that $\E_{\bx \sim \normald{m}}[h(\bx)y] \geq 1/(6C) = \E_{\bx \sim \normald{m}}[f(\bx)g(\bx)]/6 \geq \eps/6$ 
from Proposition~\ref{prop:duality-real-sq}.

Note that if $D' = \normald{n} \times \mathcal{U}(\{ \pm 1 \})$, then $\E_{(\bx,y) \sim D'}[h(\bx)y]=0$. 
Therefore, by performing a query of tolerance $\Omega(\eps)$ for the correlation of $h$ with the labels, 
we can distinguish between the two cases of our hypothesis testing problem. 
By Proposition~\ref{prop:testing-lower-bound}, this requires either $2^{n^{\Omega(1)}}$ queries 
or queries of tolerance $n^{-\Omega(d)}$.
\end{proof}

 \section{Applications for Classes of Real-Valued Functions}\label{app:real}

\subsection{ReLU Activation}
The class of Rectified Linear Unit (ReLU) functions consists of all functions 
of the form $\relu(\langle \bw, \bx \rangle)$, where $\bw \in \R^n$ is any vector 
with $\snorm{2}{\vec w}=1$ and $\relu: \R \to \R$ is defined as $\relu(t) = \max\{0, t\}$.

Upper and lower bounds for agnostically learning ReLUs were given in~\cite{GGK20, DKZ20}. 
\cite{DKZ20} established an SQ lower bound of $n^{\Omega(1/\eps^{c})}$, for some
constant $c>0$. This constant $c$ was not explicitly calculated in \cite{DKZ20},
but can be shown to be approximately $1/40$.
\cite{GGK20} gave an SQ lower bound of $n^{\Omega(1/\eps^{1/36})}$ for this problem.
We note that~\cite{GGK20} considered 
a correlational type of guarantee, i.e., finding a hypothesis whose correlation 
with the labels is within $\eps$ of the optimal, as opposed to $L^2$-error.  
For this correlational guarantee, the upper bound of~\cite{GGK20} is an 
$L^2$-regression algorithm with complexity $n^{O(\eps^{-4/3})}$, 
and the lower bound states that any SQ algorithm needs to perform queries with tolerance 
$\tau<n^{-\Omega(\eps^{-1/12})}$ or at least $2^{n^{\Omega(1)}}\eps$ queries. 
Furthermore,~\cite{GGK20} showed that any agnostic learner with the square loss guarantee 
can be run with increased accuracy to satisfy the correlational guarantee. 
This reduction costs a ``third root'' in the exponent, yielding an $n^{\Omega(\eps^{-1/36})}$ 
SQ lower bound for the square loss guarantee. As a note,~\cite{GGK20} 
assumes bounded labels. In this setting, agnostically learning within $L^2$-error $\opt+\eps$ 
is equivalent to agnostically learning in squared $L^2$-error $\opt+\eps'$, 
for $\eps'=\Theta(\eps)$.

\medskip

Given the context of prior work, we can now present our results.
To apply our generic lower bound theorems, we bound from below the degree of any polynomial 
that $\eps$-approximates the univariate ReLU function. 
This can be done by appealing to a known powerful theorem 
from the approximation theory literature by Ganburg~\cite{ganzburg2002limit,ganzburg2008limit}.
This result can be used to derive tight polynomial degree lower bounds for the $\relu$ function
and the sign function (see Appendix~\ref{ssec:ganz-lb}). 

Let $A_{\sigma} (f)_p = \inf_{g \in B_\sigma} \snorm{p}{f-g}$, where $B_\sigma$, $\sigma>0$ is the class of all entire functions of exponential type $\sigma$, i.e., the class consisting of every entire function $g$ such that for every $\eps>0$ there exists a $C$ for which $|g(z)| \leq C e^{\sigma(1+\eps)|z|}$.

\begin{fact}\label{fct:ganzburg}
For any function $f: \R \to \R$ of polynomial growth
\begin{equation*}
\lim_{n \to \infty} \left( \frac{b_n}{\sigma} \right)^{1/p} \inf_{p \in \mathcal{P}_n} \snorm{p}{f\left( \frac{b_n}{\sigma} x \right) - p(x)} = A_{\sigma} (f)_p \;,
\end{equation*}
where $b_n=2 \sqrt{n}$, $p \in [1,2]$ and $A_{\sigma} (f)_p$ is the error 
of the best approximation of $f$ by entire functions of exponential type $\sigma$ in $L^p(\R)$.
\end{fact}

As an immediate corollary, we obtain:

\begin{corollary} \label{cor:relu-lb}
Let $f: \R \to \R$ be the ReLU function $\relu(t)=\max\{0,t \}$ and $p \in [1,2]$. 
The minimum integer $d$ for which there exists a degree-$d$ polynomial $P : \R \to \R$ 
such that $\snorm{p}{\relu-P} \leq \eps$ is $d = \Theta\left(\eps^{-\frac{2}{1+1/p}}\right)$.
\end{corollary}

Therefore, Theorems~\ref{them:real-val-lb-csq} and~\ref{them:real-val-lb-sq} 
imply a complexity of at least $n^{\Omega(\eps^{-4/3})}$ for any agnostic CSQ learner; 
and $n^{\Omega(\eps^{-1})}$ for any agnostic SQ learner respectively.

\subsection{Sigmoid Activation}

\subsubsection{CSQ Lower Bound}
We now let $f$ be the standard sigmoid function, defined as $f(t) = 1/(1+e^{-t})$, $t \in \R$. 
We first focus on bounding the degree of polynomials that approximate $f$ in $\lspace{2}$-norm. 
This can be done via Hermite analysis, in particular, based on the fact that the polynomial 
of degree $d$ being closest to $f$ in $\lspace{2}$-norm is the truncated 
Hermite expansion $p_d(t) = \sum_{i=0}^d \hat{f}(i) H_i(t)$. 
The error of this approximation is $\snorm{2}{p_d - f}^2 = \sum_{i=d+1}^\infty \hat{f}^2(i)$. 
For the asymptotic behavior of the Hermite coefficients, we use the following 
fact (see~\cite{goel2020superpolynomial} and the references therein).

\begin{fact}[Lemma A.9 from~\cite{goel2020superpolynomial}]\label{fact:sigmoid}
Let $f:\R \to \R$ be the standard sigmoid function $f(t) = 1/(1+e^{-t})$ 
and $\hat{f}(i)$ be its Hermite coefficients for $i \in \Z_+$. 
Then, $\hat{f}(0)=0.5, \hat{f}(2i)=0$  and $\hat{f}(2i-1)=e^{-\Theta(\sqrt{i})}$, for $i \geq 1$.
\end{fact}

From this fact, we get the bound on the $\lspace{2}$-error of the best polynomials of degree $d$.

\begin{corollary}[$L^2$-Degree Lower Bound for Sigmoid]\label{cor:sigmoid-l2}
Let $f:\R \to \R$ be the standard sigmoid function $f(t) = 1/(1+e^{-t})$ and $d$ 
be the smallest integer for which there exists a degree-$d$ polynomial 
$p: \R \to \R$ such that $\snorm{2}{f-p}<\eps$. Then $d ={\Theta}(\log^2(1/\eps))$.
\end{corollary}
\begin{proof}
Fix a degree $k$. From Fact~\ref{fact:sigmoid}, the best $k$-degree polynomial $p_k$ achieves error
\begin{align*}
\snorm{2}{f-p_k}^2 & = \sum_{i=k+1}^\infty \hat{f}^2(i)
= \sum_{i>k, \text{$i$ odd}} e^{-\Theta(\sqrt{i})}
= \sqrt{k} e^{-\Theta(\sqrt{k})}\;.
\end{align*}
This becomes $\eps^2$ when $k$ becomes ${\Theta}(\log^2(1/\eps))$.
\end{proof}

By Theorem~\ref{them:real-val-lb-csq}, we get that any CSQ agnostic learner 
for sigmoids has complexity $n^{{\Omega}(\log^2(1/\eps))}$. 

\subsubsection{SQ Lower Bound}

The approach to derive lower bounds for the degrees of $L^1$-approximating polynomials 
will be to relate the $L^1$-norm to the $L^2$-norm and use the lower bounds for the latter. 
In particular, we will use the following fact about polynomials under the Gaussian measure.

\begin{theorem}[Hypercontractivity~\cite{Bog:98,nelson1973free}]\label{th:hypercontractivity}
If $p$ is a $d$-degree polynomial and $t>2$, then
\[
\snorm{t}{p} \leq (t-1)^{d/2} \snorm{2}{p} \;.
\]
\end{theorem}

\begin{claim}\label{cl:norm-ineq}
Let $r \in L^4(\R)$. Then, $\snorm{2}{r} \leq \snorm{1}{r}^{1/3}\snorm{4}{r}^{2/3}$.
\end{claim}
\begin{proof}
The proof follows from two applications of the Cauchy-Schwartz inequality. 
\[
\E_{t \sim \normalone}[r^2(t)] \leq \E_{t \sim \normalone}\left[|r(t)|\right]^{1/2}  \E_{t \sim \normalone}\left[|r(t)|^3\right]^{1/2} 
\leq \E_{t \sim \normalone}\left[|r(t)|\right]^{1/2} \E_{t \sim \normalone}\left[|r(t)|^2 \right]^{1/4} \E_{t \sim \normalone}\left[|r(t)|^4\right]^{1/4} \;.
\]
Rearranging the above, yields the claimed inequality.
\end{proof}

We can now show our $L^1$ polynomial degree lower bound.

\begin{theorem}[$L^1$-Degree Lower Bound for Sigmoid] \label{thm:sigmoid-l1}
Let $f:\R \to \R$ be the standard sigmoid function $f(t) = 1/(1+e^{-t})$ and $0<\eps <1$.
Any degree-$d$ polynomial $p:\R \to \R$ that satisfies $\snorm{1}{f-p}<\eps$ 
must have $d = \Omega(\log(1/\eps ))$. \end{theorem}
\begin{proof}
Let $p:\R \to \R$ be a degree-$d$ polynomial such that $\snorm{1}{f-p}<\eps$. 
Using Theorem~\ref{th:hypercontractivity} with $t=4$ and then Claim~\ref{cl:norm-ineq} 
with $r(t)=p(t)$, we get that
\begin{align*}
\snorm{4}{p} \leq 3^{d/2} \snorm{2}{p}  \leq 3^{d/2} \snorm{1}{p}^{1/3} \snorm{4}{p}^{2/3} \;.
\end{align*}
After dividing both sides by $\snorm{4}{p}^{2/3}$, we have that 
$\snorm{4}{p} \leq 3^{3d/2} \snorm{1}{p}$. Furthermore, using the triangle inequality, 
$\snorm{1}{p} \leq \eps + \snorm{1}{f} = O(1)$. Therefore, $\snorm{4}{p} \leq 2^{O(d)}$. 
Furthermore, Claim~\ref{cl:norm-ineq} for $r(t)=f(t)-p(t)$ gives
\begin{align*}
\snorm{2}{f-p} \leq \snorm{1}{f-p}^{1/3}\snorm{4}{f-p}^{2/3} 
\leq \eps^{1/3}  2^{O(d)} \;.
\end{align*}
On the other hand, for the $L^2$-error we have that 
$\snorm{2}{f-p} \geq \sqrt{d}e^{-\Theta(\sqrt{d})}$ (Corollary~\ref{cor:sigmoid-l2}).
Combining the two bounds, it follows that $d = \Omega(\log(1/\eps ))$.
\end{proof}

\nnew{We note that \cite{GGK20} showed an $n^{\Omega(\log^2(1/\eps))}$ 
SQ lower bound for the correlational guarantee.}

\clearpage

\bibliographystyle{alpha}
\bibliography{allrefs}

\newpage

\appendix

\section{Omitted Background} \label{sec:app-background}

\subsection{Correlational Statistical Query (CSQ) Model} \label{ssec:app-CSQ}
For some of our lower bounds in the real-valued setting,
we consider \emph{correlational} or inner product queries.
The CSQ model is a restriction of the SQ model, where the algorithm is
allowed to choose any bounded query function, and obtain estimates for its
correlation with the labels. 

Specifically, for $f,h:X \to \R$ and a distribution $D_{\bx}$ over the domain $X$, we denote by
$\langle f, h\rangle_{D_{\bx}}$ the quantity $\E_{\bx \sim D_{\bx}}[f(\bx)h(\bx)]$
and refer to it as the correlation of $f$ and $h$ under $D_{\bx}$.
While it is commonly assumed that the query function $h$  is pointwise bounded,
it is in fact sufficient to assume that it has bounded $\lspace{2}$-norm.
If $D$ is the joint distribution on points and labels, a correlational query takes $h$ 
and a parameter $t>0$, and outputs a value $v \in [\E_{(\bx,y)\sim D}[h(\bx)y]-\tau, \E_{(\bx,y)\sim D}[h(\bx)y]+\tau]$.

\noindent Similarly to the general SQ model, we consider the following notions
of statistical dimension.

\begin{definition}[Correlational Statistical Query Dimension] \label{def:csq-dim}
For $\beta,\gamma >0$, a probability distribution $D_{\bx}$ over domain $X$ 
and a family $\cal C$ of functions $f:X \to \R$,
let $s$ be the maximum integer for which there exists a finite set of functions
$\{f_1, \ldots, f_s\} \subseteq \mathcal{C}$ such that
$| \E_{\bx \sim D_{\bx}}[f_i^2(\bx)] | \leq \beta$ for all $i \in[s]$, 
and  $| \E_{\bx \sim D_{\bx}}[f_i(\bx) f_j(\bx)] | \leq \gamma$ for all $i,j \in [s]$ with $i\neq j$.
We define the \emph{Correlational Statistical Query Dimension} with pairwise correlations
$(\gamma, \beta)$ of $\mathcal{C}$ to be $s$ and denote it by
$\mathrm{CSD}_{D_{\bx}}(\mathcal{C},\gamma,\beta)$.
\end{definition}

\begin{definition}[Average Correlational Statistical Query Dimension] \label{def:acsq-dim}
Let $\rho>0$, let $D_{\bx}$ be a probability distribution over some domain $X$, 
and let $\cal C$ be a family of functions $f:X \to \R$. We define the average pairwise correlation 
of functions in $\mathcal{C}$ to be 
$\rho(\mathcal{C})=\frac{1}{|\mathcal{C}|^2} \sum_{g,r \in \mathcal{C}} |\E_{\bx \sim D_{\bx}}[g(\bx) r(\bx)]|$. 
The \emph{Average Correlational Statistical Query Dimension} of $\mathcal{C}$ relative to $D_{\bx}$ 
with parameter $\gamma$, denoted by $\mathrm{CSDA}_{D_{\bx}}(\mathcal{C},\gamma)$, 
is defined to be the largest integer $s$ such that  every subset $\mathcal{C}' \subseteq \mathcal{C}$ 
of size $|\mathcal{C}'| \geq |\mathcal{C}|/s$, satisfies $\rho(\mathcal{C}') \geq \rho$.
\end{definition}

In most of the cases, it suffices to bound the correlational statistical query dimension, 
since by simple calculations this implies a bound on the average statistical query dimension.

\begin{lemma} \label{lem:convert-SD-to-SDA}
Let $\mathcal{C}$ be a class of functions and $D_{\bx}$ be a distribution 
and suppose that $\mathrm{CSD}_{D_{\bx}}(\mathcal{C},\gamma,\beta) = d$, 
for some $\gamma,\beta>0$. Then, for all $\gamma'>0$, we have that 
 $\mathrm{CSD}_{D_{\bx}}(\mathcal{C},\gamma+\gamma') \geq d\gamma'/(\beta-\gamma)$.
\end{lemma}

The following result~\cite{Szorenyi09,goel2020superpolynomial}
relates the Average Correlational SQ dimension of a concept class
with the complexity of any CSQ algorithm for the class.

\begin{lemma}[Theorem B.1 in~\cite{goel2020superpolynomial}] \label{lem:gen-csq-lb}
Let $D_{\bx}$ be a distribution over a  domain $X$ and let $\mathcal{C}$ be a real-valued concept class over $X$ 
such that $0 \in \mathcal{C}$, and $\snorm{2}{f} \geq \eta$ for all $f \in \mathcal{C}, f \not\equiv 0$. 
Suppose that for some $\gamma>0$ we have $s = \mathrm{CSDA}_{D_{\bx}}(\mathcal{C},\gamma)$.
Any CSQ algorithm that outputs a hypothesis $h$ such that $\snorm{2}{h-f} < \eta$ needs at least $s/2$ queries 
or queries of tolerance $\sqrt{\gamma}$.
\end{lemma}

\subsection{Preliminaries: Multilinear Algebra} \label{app:multilinear_algebra}
Here we introduce some multilinear algebra notation.
An order $k$ tensor
$\matr A$ is an element of the $k$-fold tensor product of subspaces $\matr A
    \in \mathcal{V}_1 \otimes \ldots \otimes \mathcal{V}_k$.
We will be exclusively working with subspaces of $\R^d$ so a tensor $A$ can
be represented by a sequence of
coordinates, that is $A_{i_1,\ldots,i_k}$.   The
tensor product of a order $k$ tensor $\matr A$ and an order $m$ tensor
$\matr B$ is an order $k + m$ tensor defined as
$(\matr A \otimes \matr B)_{i_1,\ldots, i_k,j_1,\ldots,j_m} =
    \matr A_{i_1,\ldots,i_k} \matr B_{j_1,\ldots, j_m}$.  We are also going to use capital letters for
multi-indices, that is tuples of indices $I = (i_1,\ldots, i_k)$.
We denote by $E_i$ the multi-index that has $1$ on its $i$-th co-ordinate and
$0$ elsewhere. For example the previous tensor product can be denoted
as $\matr A_I \matr B_J$.
To simplify notation we are also going to use Einstein's summation where we
assume that we sum over repeated indices in a product of tensors.  For example
if $\matr A \in \R^d \otimes \R^d$, $\vec v \in \R^d$, $\vec u \in \R^d$ we
have $\sum_{i,j=1}^d \matr \bv_i \vec{u}_j \matr A_{ij} = \bv_i \vec{u}_j
    \matr A_{ij}$.  We define the dot product of two tensors (of the same order) to be
$\langle \matr A, \matr B \rangle = \matr A_{i_1,\ldots,i_k} \matr
    B_{i_1,\ldots, i_k} = \matr A_I \matr B_I$.  We also denote the $\ell_2$-norm
of a tensor by $\snorm{2}{\matr A} = \sqrt{\dotp{\matr A}{\matr A}}$.  We
denote by $\matr A(\vec X)$ a function that maps the tensor $\vec X$ to a
tensor $\matr A (\vec X)$.
Let $\mathcal{V}$ be a vector space and let $\matr A(\vec x): \R^d \to
    {\mathcal{V}}^{\otimes k}$ be a tensor valued function.  We denote by
$\partial_i \matr A(\vec x)$ the tensor of partial derivatives of $A(\vec x)$,
$\partial_i \matr A(\vec x) = \partial_i \matr A_J(\vec x) $ is a tensor of
order $k+1$ in $\mathcal{V}^{\otimes k} \otimes \R^d$.  We also denote this tensor
$
    \nabla \matr A(\vec x) = \partial_i \matr A_J(\vec x).
$
Similarly we define higher-order derivatives, and we denote
$$ \nabla^m \matr A(\vec x) = \partial_{i_1}  \ldots \partial_{i_m}
    \matr A_J(\vec x) \in \mathcal{V}^{\otimes k} \otimes (\R^d)^{\otimes m} \;.
$$

\subsection{Basics of Hermite Polynomials} \label{app:hermite_polynomials}
We are also going to use the Hermite polynomials that form an orthonormal system
with respect to the Gaussian measure.
While, usually one considers the probabilists's or physicists' Hermite polynomials,
in this work we define the \emph{normalized} Hermite polynomial of degree $i$ to be
\(
H_0(x) = 1, H_1(x) = x, H_2(x) = \frac{x^2 - 1}{\sqrt{2}},\ldots,
H_i(x) = \frac{He_i(x)}{\sqrt{i!}}, \ldots
\)
where by $He_i(x)$ we denote the probabilists' Hermite polynomial of degree
$i$.  These normalized Hermite polynomials form a complete orthonormal basis
for the single dimensional version of the inner product space $\lspace{2}$. To
get an orthonormal basis for $\lspace{2}$, we use a multi-index $J\in
    \N^d$ to define the $d$-variate normalized Hermite polynomial as $H_J(\vec x) =
    \prod_{i=1}^d H_{v_i}(\x_i)$.  The total degree of $H_J$ is $|J| = \sum_{v_i \in
        J} v_i$.  Given a function $f \in \lspace{2}(\R)$ we compute its Hermite coefficients as
\(
\hat{f}(J) = \E_{\vec x\sim \normald{n}} [f(\vec x) H_J(\vec x)]
\)
and express it uniquely as
\(
\sum_{J \in \N^n} \hat{f}(J) H_J(\vec x).
\)
For more details on the Gaussian space and Hermite Analysis (especially from
the theoretical computer science perspective), we refer the reader to
\cite{Don14}.  Most of the facts about Hermite polynomials that we use in this
work are well known properties and can be found, for example, in \cite{Sze67}.

We denote by $f^{[k]}(x)$ the degree $k$ part of the Hermite expansion of $f$,
$f^{[k]} (\vec x) = \sum_{|J| = k} \hat{f}(J)\cdot H_J(\vec x)$.
We say that a polynomial $q$ is harmonic of degree $k$ if it is
a linear combination of degree $k$ Hermite polynomials, that is $q$ can be
written as
$$ q(\vec x) = q^{[k]}(\vec x) = \sum_{J: |J| = k} c_J H_J(\vec x) \;.
$$

For a single dimensional Hermite polynomial it holds
$H_m'(x) = \sqrt{m} H'_{m-1}(x)$.  Using this we obtain that for a multivariate
Hermite polynomial $H_M(\vec x)$, where $M = (m_1,\ldots, m_n)$ it holds
\begin{equation}
\label{eq:hermite_nabla}
\nabla H_M(\vec x) = \sqrt{m_i} H_{M - E_i}(\vec x) \in \R^n,
\end{equation}
where $E_i = \vec e_i$ is the multi-index that has $1$ position $i$ and $0$
elsewhere.  From this fact and the orthogonality of Hermite polynomials
we obtain
\begin{equation}
\label{eq:hermite_nabla_dot}
\E_{\vec x \sim \normald{n}}[ \dotp{\nabla H_M(\vec x)}{\nabla H_L(\vec x)}]
= |M| \delta_{M, L}.
\end{equation}

The following fact gives us a formula for the inner product of
\begin{fact}\label{fct:harmonic_nabla_dot}
Let $p, q:\R^n\to\R$ be harmonic polynomials of degree $k$.  Then
$$
    \E_{\vec x \sim \normald{n}}\left[\dotp{\nabla^{\ell} p(\vec x)}{\nabla^{\ell} q(\vec x)}\right]
    =
    k(k-1)\ldots(k-\ell+1)
    \E_{\vec x \sim \normald{n}}[p(\vec x) q(\vec x)] \;.
$$
In particular,
$$
\dotp{\nabla^{k} p(\vec x)}{\nabla^{k} q(\vec x)}  =
k!  \E_{\vec x \sim \normald{n}}[p(\vec x) q(\vec x)] \;.
$$
\end{fact}
\begin{proof}
Write $p(\vec x) = \sum_{M: |M| = k} b_M H_M(\vec x)$ and
$q(\vec x) = \sum_{M: |M| = k} c_M H_M(\vec x)$.
Since the Hermite polynomials are orthonormal we obtain $\E_{\vec x \sim
        \normald{n}}[p(\vec x) q(\vec x)] = \sum_{M: |M| = k} c_M b_M$.
Now, using Equation~\eqref{eq:hermite_nabla} iteratively we obtain
$$
    \E_{\vec x \sim \normal}
    \left[\dotp{\nabla^{\ell} H_M(\vec x)}{\nabla^{\ell} H_L(\vec x)}\right]
    = k (k-1)\ldots (k-\ell + 1) \delta_{M, L}.
$$
Using this equality we obtain
\begin{align*}
\E_{\vec x \sim \normal}\left[\dotp{\nabla^{\ell} p(\vec x)}{\nabla^{\ell} q(\vec x)}\right]
 & =
\E_{\vec x \sim \normal}\left[\dotp{\sum_M b_M \nabla^{\ell} H_M(\vec x)}{ \sum_L c_L \nabla^{\ell} H_L(\vec x) }\right] \\
& =
\sum_{M,L} b_M c_L \E_{\vec x \sim \normal}\left[\dotp{\nabla^{\ell} H_M(\vec x)}{\nabla^{\ell} H_L(\vec x)}\right] \\
& = \sum_{M,L} b_M c_L  k (k-1)\ldots (k-\ell + 1) \delta_{M, L} \\
& = k (k-1)\ldots (k-\ell + 1) \E_{\vec x \sim \normal}[p(\vec x) q(\vec x)].\qquad \qedhere
\end{align*}
\end{proof}
Observe that for every harmonic polynomial $p(\x)$ of degree $k$ we have that
$\nabla^{k} p(\vec x)$ is a symmetric tensor of order $k$.  Since the degree of
the polynomial is $k$ and we differentiate $k$ times this tensor no longer
depends on $\vec x$.  Using Fact~\ref{fct:harmonic_nabla_dot}, we observe that
this operation (modulo a division by $\sqrt{k!}$) preserves the $L^2$-norm of
the harmonic polynomial $p$, that is 
$\E_{\vec x \sim \normald{n}}[p^2(\vec x)] = \snorm{2}{\nabla^{k} p(\vec x)}^2/k!$.

\section{Omitted Proofs from Section~\ref{sec:app-bool}}

\subsection{Low-Degree Polynomial Approximation to the Sign Function} \label{ssec:ganz-lb}

By selecting $f(t) = \sign(t)$ and $p=1$ in Fact~\ref{fct:ganzburg}, 
we get that any polynomial that achieves error
at most $\eps$ with respect to the $\lspace{1}$-norm 
must have degree at least $\Omega(1/\eps^2)$.

\begin{corollary}\label{cor:sign-lb-deg}
Let $f: \R \to \{\pm 1\}$ with $f(t)=\sign(t)$.
Any polynomial $p : \R \to \R$ satisfying $\snorm{1}{f-p} \leq \eps$ must have degree
$d=\Omega(1/\eps^2)$.
\end{corollary}

\subsection{Proof of Lemma~\ref{lem:gns-lb-inters}} \label{ssec:gns-lb-inters}
We restate the lemma below.
\begin{lemma}
There exists an intersection of $k$ halfspaces on $\R^k$, $f: \R^k \to \{\pm 1\}$ such that
$\gns_{\eps} (f) = \Omega(\sqrt{\eps \log k})$.
\end{lemma}
\begin{proof}
We will exhibit a family of $k$ halfspaces whose intersection has the claimed Gaussian noise sensitivity.
In particular, these halfspaces will be orthogonal.
For $i\in[k]$, let $f_i : \R^n \to \{ \pm 1 \}$ with $f_i(\bx) = \sign(-\langle \be_i, \bx \rangle + \theta)$,
where $\be_i$ is the vector having $1$ in the $i$-th coordinate and $0$ elsewhere,
and $\theta>0$ is the bias. That is,  $f_i$ is $1$ if and only if the $i$-th coordinate is less than $\theta$.

Fix an index $i \in  [k]$. The Gaussian noise sensitivity of a single halfspace is
$\gns_\eps(f_i) = \Omega(e^{-\frac{\theta^2}{2(1-\eps/2)}}\sqrt{\eps})$ (see, e.g.,  \cite[Lemma 3.4]{diakonikolas2015noise} for a proof). Let $\bx,\by$ be two $(1{-\eps})$-correlated $n$-dimensional standard Gaussian random variables.
Then, the inner products $\langle \be_i, \bx \rangle$ and $\langle \be_i, \by \rangle$ are $(1{-\eps})$-correlated univariate Gaussians.
Since the Gaussian noise sensitivity of $f_i$ is proportional to the probability that $\langle \be_i, \bx \rangle < \theta <\langle \be_i, \by \rangle$, we have that
\begin{equation*}
\pr_{(\bx,\by) \sim \normald{n}^{1-\eps}}[\langle \be_i, \bx \rangle < \theta <\langle \be_i, \by \rangle] = \Omega(e^{-\frac{\theta^2}{2(1-\eps/2)}}\sqrt{\eps})\;.
\end{equation*}
Let $\theta$ be the threshold for which $\pr_{\bx \sim \normald{n}}[\langle \be_i, \bx \rangle > \theta]=1/k$.
The standard bound for the Gaussian tail is $\pr_{\bx \sim \normald{n}}[\langle \be_i, \bx \rangle > \theta] = \Theta(e^{-\theta^2/2}/\theta)$. Therefore, for the $\theta$ that we selected it holds
$\pr_{(\bx,\by) \sim \normald{n}^{1-\eps}}[\langle \be_i, \bx \rangle < \theta <\langle \be_i, \by \rangle] =
    \Omega(\theta\sqrt{\eps}/k) = \Omega(\sqrt{\eps \log k}/k )$.

Let $f : \R^n \to \{ \pm 1 \}$ be 1 if and only if $f_i$ is 1 for all $i \in [k]$. Then, we have that
\begin{align*}
\gns_{\eps}(f)
 & = 2 \pr_{(\bx,\by) \sim \normald{n}^{1-\eps}}[f(\bx)=1, f(\by)=-1] =
\pr_{\bx \sim \normald{n}}[f(\bx)=1] - \pr_{(\bx,\by) \sim \normald{n}^{1-\eps}}[f(\bx)=f(\by)=1]                             \\
 & = \left( 1- \frac{1}{k} \right)^k - \left(1- \frac{1}{k} - \Omega\left(\frac{\sqrt{\eps \log k}}{k} \right) \right)^k  \;,
\end{align*}
where the $k$-th powers are due to the fact that $\langle \be_i, \bx \rangle$
and $\langle \be_j, \bx \rangle$ are independent for $i\neq j$.
We can use the Taylor expansion to show that the above difference is $\Omega(\sqrt{\eps \log k})$.
Let the function $h(t) = (1-1/k+t)^k$. By Taylor's theorem, $h(0)-h(t) = -h'(0)t-h''(\xi)t^2/2$,
for some $\xi$ between $t$ and $0$. By calculating the derivatives,
setting $t = - \Omega(\sqrt{\eps \log k}/k)$ and noting that the second term of the approximation is less than the first one,
we get that $h(0) - h(t) =  \Omega\left(\frac{\sqrt{\eps \log k}}{k}\right)k\left(1-\frac{1}{k}\right)^{k-1}$.
\end{proof}
 \section{Duality in Infinite-Dimensional LP}\label{app:duality}
We start with some basic definitions.
\paragraph{$L^p$ space}
Let $({ X},\mathcal{A},\mu)$ be a measure space and $1\leq p< \infty$.
We will typically take $X = \R^n$, $n \in \Z_+$,
and $\mu$ be the Gaussian measure, unless otherwise specified.
For a function $f: { X} \to \R$, the $L^p$-norm of $f$
under $\normald{n}$
is defined as $\|f\|_p \eqdef \left( \int_{{ X}} |f|^p \d\, \mu \right)^{1/p}$.
For the special case where $p=\infty$, the $L^{\infty}$-norm of $f$
is defined as the essential supremum of $f$ on ${ X}$, i.e.,
$\|f\|_\infty  \eqdef \inf\{ a \in \mathbb{R} : \mu \{\bx \in { X} : f(\bx)>a  \}=0 \}$.
The vector space $L^p({ X},\mu)$ consists of all functions $f: { X} \to \R$
with $\|f\|_p < \infty$. We will typically use the shortened notation $L^p(\R^n)$ for $L^p(\R^n,\normald{n})$.

\paragraph{Dual Norms} Consider a vector space $V$ with inner product
$\langle \cdot, \cdot \rangle$ and a norm $ \snorm{}{\cdot}$ on $V$.
The {\em dual norm} $\snorm{*}{f}$, $f \in V$, is defined as
$\snorm{*}{f} = \sup \{\langle f,h \rangle : \snorm{}{h} \leq 1  \}$.
H\"older's inequality states that for any $f, h \in V$ it holds
$\langle f,h \rangle \leq \snorm{}{f} \snorm{*}{h} $.

\paragraph{Basics on Duality of Infinite-Dimensional LPs}
For succinctness, we will use the following notation.
We use $(\tilde{h}, t)$ for the inequality $\E_{\bx \sim \normald{m}}[g(\x) \tilde{h}(\x)] + t  \leq 0$,
where $\tilde{h}\in \cal X$ and $t\in \R$.
Here $\cal X$ is an appropriate space of functions that in our context will be $\lspace{p}(\R^m)$.
Let $\cal{S}$ be the set of all such tuples that describe the target LP.
For the set $\cal{S}$, the closed convex cone over ${\cal X}\times \R$ is the smallest closed set
$\cal{S}_+$ satisfying the following: if $A\in \cal{S}_+$ and $B\in \cal{S}_+$ then $A+B \in \cal{S}_+$;
and if $A\in \cal{S}_+$ then $\lambda A \in \cal{S}_+$, for all $\lambda \geq 0$.

In our arguments, we need to prove that there exists a function
$g:\R^m\to \R$, such that for any function
$h\in \lspace{p}(\R^m)$ and at most $(d-1)$-degree polynomial
$P:\R^m\to\R$, it holds
\begin{empheq}[left=(\ast)\empheqlbrace]{align}\label{sys:primal-boolean}
&-\E_{\bx \sim \normald{m}}[g(\x)f(\x)]+c \leq 0 & 0<&c<\|f\| \notag\\
&\E_{\bx \sim \normald{m}}[P(\x)g(\x)]  =0 & \forall P & \in {\cal P}_{d-1}^m \notag\\
&\E_{\bx \sim \normald{m}}[g(\x)h(\x)] -\|h\|_p\leq 0 &\forall h&\in\lspace{p}(\R^m) \notag
\end{empheq}
This is in fact an infinite dimensional linear system with respect to the unknown function $g\in (\lspace{1}(\R^m))^\ast=L^{\infty}(\R^m)$, for $p=1$ and $g\in \lspace{p/(p-1)}(\R^m)$ for $1\leq p<\infty$. We are going to denote $\cal X$ the metric space $\lspace{p}(\R^m)$.

For succinctness, we will use the following notation.
We use $(\tilde{h}, t)$ for the inequality $\E_{\bx \sim \normald{m}}[g(\x) \tilde{h}(\x)] + t  \leq 0$,
where $\tilde{h}\in \cal X$ and $t\in \R$. Moreover, let $\cal S$ be the set that contains all such tuples that describe the target system. For the set $\cal S$, the closed convex cone over ${\cal X}\times R$ is the smallest closed set ${\cal S}_+$ satisfying, if $A\in {\cal S}_+$ and $B\in {\cal S}_+$ then $A+B \in {\cal S}_+$ and, if $A\in {\cal S}_+$ then $\lambda A \in {\cal S}_+$ for all $\lambda\geq 0$. Note that the ${\cal S}_+$ contains the same feasible solutions as $\cal S$. The set ${\cal S}=\{(h,-\|h\|_p): h\in \lspace{p}\}\cup \{(P,0): P\in {\cal P}_{d-1}^m\}\cup\{(-f,c)\}$.

In the finite-dimensional case, we can always prove the feasibility of an LP by applying the standard
Farkas' lemma (aka theorem of the alternative). However, when the system is infinite-dimensional,
Farkas' lemma does not hold in general. We are going to use the following result from \cite{Fan68}.
\begin{lemma}[Theorem 1 of \cite{Fan68}]
If $\cal X$ is a locally convex, real separated vector space then, a linear system described by $\cal S$ for which ${\cal S}_+$ is closed is feasible (i.e., there exists a $g\in {\cal X}^*$) if and only if $(0,1)\not\in {\cal S}_+$ and ${\cal S}_+$ is closed.
\end{lemma}
One direction is trivial, but the other one needs an application of Hahn-Banach theorem which is where the assumption on $\cal X$ to be a separated space is used.
\begin{corollary} \label{lem:duality}  If ${\cal X}= \lspace{p}$ for $1\leq p<\infty$ then, the $LP$ described by $\cal S$ is feasible if only if  $(0,1)\not\in {\cal S}_+$.

\end{corollary}
\begin{proof}
For ${\cal X}= \lspace{p}$ and $1\leq p<\infty$, $\cal X$ is a locally convex, real separated vector space.
Finally, we need to prove that the set ${\cal S}_+$ is closed.

We begin by finding an explicit representation of ${\cal S}_+$. It is not hard to see that
$$
    {\cal S}_+ = \{ (P+h-yf,\|h\|_p+yc+t) : P \in {\cal P}_{d-1}^m, h \in \lspace{p}(\R^m), y,t\in\R, y,t\geq 0\}.
$$
We will show that this is closed, by showing that it is closed under limits. In particular, suppose that there is some sequence $(P_i,h_i,y_i,t_i)$ so that $(P_i+h_i-y_if,\|h_i\|_p+y_ic+t_i)$ converges to some limit $(\tilde h, \tilde t)$. We claim then that $(h,t)$ is in ${\cal S}_+$.

To show this, we first note that for $i$ sufficiently large $\|h_i\|_p \leq \|h_i\|_p+y_ic+t_i \leq \tilde t + 1$, and $y_ic \leq \|h_i\|_p+y_ic+t_i \leq \tilde t + 1$. Thus, for $i$ sufficiently large $\|h_i\|_p \leq \tilde t+1$ and $y_i \leq (\tilde t+1)/c$. Furthermore, for $i$ sufficiently large $\|P_i+h_i-y_if\|_p \leq \|\tilde h\|_p+1$. However, $\|P_i\|_p \leq \|P_i+h_i-y_if\|_p + \|h_i\|_p + y_i\|f\|_p$, which is bounded for $i$ sufficiently large. Therefore, since an $L^p$ ball in ${\cal P}_{d-1}^m$ is compact, by restricting to a subsequence, we can assume that the $P_i$ have some limit, say, $P$. Furthermore, since $[0,(\tilde t+1)/c]$ is compact, restricting to a further subsequence, we can assume that the $y_i$ have some limit $y$. Since $(P_i+h_i-y_if)$ have limit $\tilde h$, the $h_i$ must approach a limit $h=\tilde h -P+yf$. Finally, we note that $\|h_i\|_p+y_ic+t_i$ has limit $\tilde t$, and thus, the $t_i$ must approach a limit $t=\tilde t- yc-\|h\|_p$. In particular, we must have $t\geq 0$.

However, given the above we have that
$$
    (\tilde h, \tilde t) = (P+h-yf,\|h\|_p+yc+t) \in {\cal S}_+.
$$
This completes our proof

\end{proof}

\end{document}